\documentclass[english, 11pt]{article}
\usepackage[T1]{fontenc}
\usepackage[latin9]{inputenc}
\usepackage{lmodern}
\usepackage{geometry}
\usepackage{amsthm}
\usepackage{amsmath}
\usepackage{amssymb}
\usepackage{esint}
\usepackage[numbers, sort]{natbib}
\usepackage{algorithm, algorithmic}
\usepackage{authblk}
\usepackage{color}
\usepackage{caption}
\usepackage{subcaption}
\usepackage{wrapfig}
\usepackage{boldline}
\usepackage{graphicx}
\usepackage{array}
\usepackage{multirow}
\usepackage{setspace}
\makeatletter
  \theoremstyle{plain}
  \newtheorem{prop}{\protect\propositionname}
  \theoremstyle{plain}
  
  \theoremstyle{definition}
  
  \theoremstyle{plain}
  \newtheorem{cor}{\protect\corollaryname}
  \theoremstyle{example}
  \newtheorem{example}{\protect\examplename}
 \theoremstyle{remark}

\theoremstyle{plain}
\newtheorem{fact}{\protect\factname}
\theoremstyle{plain}
\newtheorem{assumption}{\protect\assumptionname}

\newtheoremstyle{PropositionNum}
        {\topsep}{\topsep}              
        {\itshape}                      
        {}                              
        {\bfseries}                     
        {.}                             
        { }                             
        {\thmname{#1}\thmnote{ \bfseries #3}}
\theoremstyle{PropositionNum}
\newtheorem{propn}{Proposition}

\newtheoremstyle{LemmaNum}
        {\topsep}{\topsep}              
        {\itshape}                      
        {}                              
        {\bfseries}                     
        {.}                             
        { }                             
        {\thmname{#1}\thmnote{ \bfseries #3}}
\theoremstyle{LemmaNum}

\newtheoremstyle{DefinitionNum}
        {\topsep}{\topsep}              
        {\itshape}                      
        {}                              
        {\bfseries}                     
        {.}                             
        { }                             
        {\thmname{#1}\thmnote{ \bfseries #3}}
\theoremstyle{DefinitionNum}

\makeatother
\newcommand{\Y}{\mathcal{Y}}
\newcommand{\D}{\mathcal{D}}
\newcommand{\E}{\mathbb{E}}
\newcommand{\Prob}{\mathbb{P}}
\newcommand{\hist}{\mathcal{F}_{t}}

\newcommand{\A}{\mathcal{A}}

\newcommand{\DKL}{D_{\rm KL}}

\usepackage{babel}
  \providecommand{\assumptionname}{Assumption}
  \providecommand{\definitionname}{Definition}
  \providecommand{\lemmaname}{Lemma}
  \providecommand{\propositionname}{Proposition}
\providecommand{\corollaryname}{Corollary}
  \providecommand{\examplename}{Example}
\providecommand{\factname}{Fact}

\usepackage{tikz}
\usetikzlibrary{calc,trees,positioning,arrows,chains,shapes.geometric,%
    decorations.pathreplacing,decorations.pathmorphing,shapes,%
    matrix,shapes.symbols}

\tikzset{
>=stealth',
  punktchain/.style={
    rectangle,
    rounded corners,
    draw=black, very thick,
    text width=10em,
    minimum height=3em,
    text centered,
    on chain},
  line/.style={draw, thick, <-},
  element/.style={
    tape,
    top color=white,
    bottom color=blue!50!black!60!,
    minimum width=8em,
    draw=blue!40!black!90, very thick,
    text width=10em,
    minimum height=3.5em,
    text centered,
    on chain},
  every join/.style={->, thick,shorten >=1pt},
  decoration={brace},
  tuborg/.style={decorate},
  tubnode/.style={midway, right=2pt},
}

\begin{document}

\title{Learning to Optimize Via Information-Directed Sampling}

\author[1]{Daniel Russo}
\author[2]{Benjamin Van Roy}
\affil[1]{Northwestern University, daniel.russo@kellogg.northwestern.edu}
\affil[2]{Stanford University, bvr@stanford.edu}

\maketitle
\begin{abstract}
We propose {\it information-directed sampling} -- a new approach to online optimization problems in which a decision-maker must balance between exploration and exploitation while learning from partial feedback.  Each action is sampled in a manner that minimizes the ratio between squared expected single-period regret and a measure of information gain: the mutual information between the optimal action and the next observation.

We establish an expected regret bound for information-directed sampling that applies across a very general class of models and scales with the entropy of the optimal action distribution. 
We illustrate through simple analytic examples how information-directed sampling accounts for kinds of information that alternative approaches do not adequately address and that this can lead to dramatic performance gains.  For the widely studied Bernoulli, Gaussian, and linear bandit problems, we demonstrate state-of-the-art simulation performance. 
\end{abstract}

\section{Introduction}

In the classical multi-armed bandit problem, a decision-maker repeatedly chooses from among a finite set of actions.  Each action generates a random reward drawn independently from a probability distribution associated with the action.  The decision-maker is uncertain about these reward distributions, but learns about them as rewards are observed.  Strong performance requires striking a balance between  {\it exploring} poorly understood actions and {\it exploiting} previously acquired knowledge to attain high rewards.  Because selecting one action generates no information pertinent to other actions, effective algorithms must sample every action many times.

There has been significant interest in addressing problems with more complex {\it information structures}, in which sampling one action can inform the decision-maker's assessment of other actions.  Effective algorithms must take advantage of the information structure to learn more efficiently.  The most popular approaches to such problems
extend {\it upper-confidence-bound} (UCB) algorithms and {\it Thompson sampling}, which were originally devised
for the classical multi-armed bandit problem.  In some cases, such as the linear bandit problem, strong performance guarantees have been established 
for these approaches.  For some problem classes, compelling empirical results have also been presented for UCB algorithms and Thompson sampling, 
as well as the knowledge gradient algorithm.  However, as we will demonstrate through simple analytic examples, these approaches can perform very poorly when 
faced with more complex information structures.  Shortcomings stem from the fact that they do not adequately account 
for particular kinds of information.

In this paper, we propose a new approach -- {\it information-directed sampling} (IDS) -- that is designed to address this.
IDS quantifies the amount learned by selecting an action through an information-theoretic measure: the mutual information between the true optimal action and the next observation.  Each action is sampled in a manner that minimizes the ratio between squared expected single-period regret and this measure of information gain.  Through this information measure, IDS accounts for kinds of information that alternatives fail to address.

As we will demonstrate through simple analytic examples, IDS can dramatically outperform UCB algorithms, Thompson sampling, and the knowledge-gradient algorithm.  Further, by leveraging the tools of our recent information theoretic analysis of Thompson sampling \cite{russo2016info}, we establish an expected regret bound for IDS that applies across a very general class of models and scales with the entropy of the optimal action distribution.  We also specialize this bound to several classes of online optimization problems, including problems with full feedback, linear optimization problems with bandit feedback, and combinatorial problems with semi-bandit feedback, in each case establishing that bounds are order optimal up to a poly-logarithmic factor.

We benchmark the performance of IDS through simulations of the widely studied Bernoulli, Gaussian, and linear bandit problems, for which UCB algorithms and Thompson sampling are known to be very effective.  We find that even in these settings, IDS outperforms UCB algorithms and Thompson sampling. This is particularly surprising for Bernoulli bandit problems, where UCB algorithms and Thompson sampling are known to be asymptotically optimal in the sense proposed by \citet{lai1985asymptotically}.

IDS solves a single-period optimization problem as a proxy to an intractable multi-period problem.  Solution of this single-period problem can itself be computationally demanding, especially in cases where the number of actions is enormous or mutual information is difficult to evaluate.  We develop numerical methods for particular classes of online optimization problems.  In some cases, our numerical methods do not compute exact or near-exact solutions but generate efficient approximations that are intended to capture key benefits 
of IDS.  There is much more work to be done to design efficient algorithms for various problem classes and we hope that our analysis and initial collection of numerical methods will provide a foundation for further developments.

It is worth noting that the problem formulation we work with, which is presented in Section \ref{sec:formulation}, is very general, encompassing not only problems with bandit feedback, but also a broad array of information structures for which observations can offer information about rewards of arbitrary subsets of actions or factors that influence these rewards.  Because IDS and our analysis accommodate this level of generality, they can be specialized to problems that in the past have been studied individually, such as those involving pricing and assortment optimization (see, e.g., \cite{rusmevichientong2010dynamic, broder2012dynamic, saure2013optimal}), though in each case, developing a computationally efficient version of IDS may require innovation.

\section{Literature review}

UCB algorithms are the primary approach considered in the segment of the stochastic multi-armed bandit literature that treats problems with dependent arms. UCB algorithms have been applied to problems where the mapping from action to expected reward is a linear \cite{dani2008stochastic, rusmevichientong2010linearly, abbasi2011improved}, generalized linear \cite{filippi2010parametric}, or sparse linear \cite{abbasi2012online} model; is sampled from a Gaussian process \cite{srinivas2012information} or has small norm in a reproducing kernel Hilbert space \cite{srinivas2012information, valko2013finite}; or is a smooth (e.g. Lipschitz continuous) model \cite{kleinberg2008multi, bubeck2011xarmed, valko2013stochastic}. Recently, an algorithm known as Thompson sampling has received a great deal of interest. \citet{agrawal2013linear} provided the first analysis for linear contextual bandit problems. \citet{russo2013nipsEluder,russo2014learning} consider a more general class of models, and show that standard analysis of upper confidence bound algorithms leads to bounds on the expected regret of Thompson sampling. Very recent work of \citet{gopalan2014thompson} provides asymptotic frequentist bounds on the growth rate of regret for problems with dependent arms. Both UCB algorithms and Thompson sampling have been applied to other types of problems, like reinforcement learning \cite{jaksch2010near, osband2013more} and Monte Carlo tree search \cite{kocsis2006bandit, bai2013bayesian}.

In one of the first papers on multi-armed bandit problems with dependent arms, \citet{agrawal1989asymptotically} consider a general model in which the reward distribution associated with each action depends on a common unknown parameter. When the parameter space is finite, they provide a lower bound on the asymptotic growth rate of the regret of any admissible policy as the time horizon tends to infinity and show that this bound is attainable. These results were later extended by \citet{agrawal1989markov} and \citet{graves1997asymptotically} to apply to the adaptive control of Markov chains and to problems with infinite parameter spaces. These papers provide results of fundemental importance, but seem to have been overlooked by much of the recent literature. 

Though the use of mutual information to guide sampling has been the subject of much research, dating back to the work of \citet{lindley1956},
to our knowledge, only two other papers \cite{villemonteix2009informational, hennig2012entropy} have used the mutual information between the optimal action and the next observation to guide action selection. Each focuses on optimization of expensive-to-evaluate black-box functions. Here, {\it black--box} indicates the absence of strong structural assumptions such as convexity and that the algorithm only has access to function evaluations, while {\it expensive-to-evaluate} indicates that the cost of evaluation warrants investing considerable effort to determine where to evaluate. These papers focus on settings with low-dimensional continuous action spaces, and with a Gaussian process prior over the objective function, reflecting the belief that ``smoother'' objective functions are more plausible than others. This approach is often called ``Bayesian optimization'' in the machine learning community \cite{brochu2009tutorial}. Both \citet{villemonteix2009informational} and \citet{hennig2012entropy} propose selecting each sample to maximize the mutual information between the next observation and the optimal solution. Several papers \cite{hernandez2014predictive, hernandez2015predictive, hernandez2015predictiveB} have extended this line of work since an initial version of our paper appeared online. The numerical routines in these papers use approximations to mutual information, and  may give insight into how to design efficient computational approximations to IDS.  

Several features distinguish our work from that of \citet{villemonteix2009informational} and \citet{hennig2012entropy}.  First, these papers focus on pure exploration problems: the objective is simply to learn about the optimal solution -- not to attain high cumulative reward. Second, and more importantly, they focus only on problems with Gaussian process priors and continuous action spaces. For such problems, simpler approaches like UCB algorithms \cite{srinivas2012information}, probability of improvement \cite{kushner1964new}, and expected improvement \cite{mockus1978application} are already extremely effective. As noted by \citet{brochu2009tutorial}, each of these algorithms simply chooses points with ``{\it potentially} high values of the objective function: whether because the prediction is high, the uncertainty is great, or both.'' By contrast, a major motivation of our work is that a richer information measure is needed to address problems with more complicated information structures. Finally, we provide a variety of general theoretical guarantees for information-directed sampling, whereas \citet{villemonteix2009informational} and \citet{hennig2012entropy} propose their algorithms as heuristics without guarantees.  Section \ref{subsec: pure exploration} shows that our theoretical guarantees extend to pure exploration problems.

%
%
%
%

The knowledge gradient (KG) algorithm uses a different measure of information to guide action selection: the algorithm computes the impact of a single observation on the quality of the decision made by a {\it greedy} algorithm, which simply selects the action with highest posterior expected reward. This measure was proposed by \citet{mockus1978application} and studied further by \citet{frazier2008knowledge} and \citet{ryzhov2012knowledge}.  KG seems natural since it explicitly seeks information that improves decision quality.  Computational studies suggest that for problems with Gaussian priors, Gaussian rewards, and relatively short time horizons, KG performs very well. However, there are no general guarantees for KG, and even in some simple settings, it may not converge to optimality. In fact, it may select a suboptimal action in {\it every} period, even as the time horizon tends to infinity. IDS also measures the information provided by a single observation, but our results imply it converges to optimality. KG is discussed further in Subsection \ref{subsec: expected improvement}. 

Our work also connects to a much larger literature on Bayesian experimental design (see \cite{chaloner1995bayesian} for a review).
\citet{Contal2014} study problems with Gaussian process priors and a method that guides exploration using the mutual information between the objective function and the next observation.  This work provides a regret bound, though, as the authors' erratum indicates, the proof of the regret bound is incorrect.
Recent work has demonstrated the effectiveness of {\it greedy} or {\it myopic} policies that always maximize a measure of the information gain from the next sample.
 \citet{jedynak2012twenty} and \citet{waeber2013bisection} consider problem settings in which this greedy policy is optimal.  Another recent line of work \cite{golovin2010near, golovin2011adaptive} shows that measures of information gain sometimes satisfy a decreasing returns property known as adaptive sub-modularity, implying the greedy policy is competitive with the optimal policy. Our algorithm also only considers the information gain due to the {\it next sample}, even though the goal is to acquire information over many periods.  Our results establish that the manner in which IDS encourages information gain  leads to an effective algorithm, even for the different objective of maximizing cumulative reward.
 
Finally, our work connects to the literature on partial monitoring. First introduced by \cite{piccolboni2001discrete} the partial monitoring problem encompasses a broad range of online optimization problems with limited or partial feedback. Recent work \cite{bartok2014partial} has focused on classifying the minimax-optimal scaling of regret in the problem's time horizon as a function of the level of feedback the agent receives. That work focuses most attention on cases where the agent receives very restrictive feedback, and in particular, cannot observe the reward their action generates. Our work also allows the agent to observe rich forms of feedback in response to actions they select, but we focus on a more standard decision-theoretic framework in which the agent associates their observations with a reward as specified by a utility function.
 
The literature we have discussed primarily  focuses on contexts where the goal is to converge on an optimal action in a manner that limits exploration costs.  
Such methods are not geared towards problems where time preference plays an important role.  A notable exception is the KG algorithm, which takes a discount
factor as input to account for time preference.  \citet{francetich2016toolkita,francetich2016toolkitb} discuss a variety of heuristics for the discounted problem.  Recent work
\cite{russo2017discount} generalizes Thompson sampling to address discounted problems.  We believe that IDS can also be extended to treat discounted problems,
though we do not pursue that in this paper.

The regret bounds we will present build on our information-theoretic analysis of Thompson sampling \citep{russo2016info}, which can be used to bound the regret of any policy in terms of its information ratio.  The information ratio of IDS is always smaller than that of TS, and therefore, bounds on the information ratio of TS provided in \citet{russo2016info} yield regret bounds for IDS.  This observation and a preliminary version of our results was first presented in a conference paper \cite{russo2014nipsIDS}.  Recent work by \citet{bubeck2015bandit} and \citet{bubeck2016multi} build on ideas from \citep{russo2016info} in another direction by bounding the information ratio when the reward function is convex and using that bound to study the order of regret in adversarial bandit convex optimization.

\section{Problem formulation}
\label{sec:formulation}
We consider a general probabilistic, or Bayesian, formulation in which uncertain quantities are modeled as random variables.
The decision-maker sequentially chooses actions$(A_t)_{t\in \mathbb{N}}$ from a finite action set $\A$ and observes the corresponding outcomes $\left(Y_{t, A_t}\right)_{t\in \mathbb{N}}$. There is a random outcome $Y_{t, a} \in \Y$ associated with each action $a\in \A$ and time $t \in \mathbb{N}$.  Let $Y_t \equiv (Y_{t,a})_{a \in \A}$ be the vector of outcomes at time $t \in \mathbb{N}$.  There is a random variable $\theta$ such that, conditioned on $\theta$, $(Y_{t})_{t\in \mathbb{N}}$ is an iid sequence. This can be thought of as a Bayesian formulation, where randomness in $\theta$ captures the decision-maker's prior uncertainty about the true nature of the system, and the remaining randomness in $Y_t$ captures idiosyncratic randomness in observed outcomes.


The agent associates a reward $R(y)$ with each outcome $y\in \mathcal{Y}$, where the reward function $R: \mathcal{Y} \rightarrow \mathbb{R}$ is fixed and known.  Let $R_{t,a} = R(Y_{t,a})$ denote the realized reward of action $a$ at time $t$.
Uncertainty about $\theta$ induces uncertainty about the true optimal action, which we denote by $A^* \in \underset{a \in \A}{\arg \max}\, \E\left[ R_{1,a} | \theta \right]$. The $T$--period {\it regret} of the sequence of actions $A_1, .., A_T$ is the random variable
\begin{equation}\label{eq: regret}
{\rm Regret}(T) := \sum_{t=1}^{T} \left( R_{t,A^*} - R_{t,A_t} \right),
\end{equation}
which measures the cumulative difference between the reward earned by an algorithm that always chooses the optimal action and actual accumulated reward up to time $T$. In this paper we study expected regret 
\begin{equation}\label{eq: expected regret}
\E \left[ {\rm Regret}(T)\right] = \E \left[\sum_{t=1}^{T}  \left( R_{t,A^*} - R_{t,A_t} \right) \right],
\end{equation}
where 
the expectation is taken over the randomness in the actions $A_t$ and the outcomes $Y_t$, and over the prior distribution over $\theta$. This measure of performance is commonly called {\it Bayesian regret} or {\it Bayes risk}.

The action $A_t$ is chosen based on the history of observations $\hist=(A_1, Y_{1,A_1},\ldots,A_{t-1}, Y_{t-1, A_{t-1}})$ up to time $t$. Formally, a \emph{randomized policy} $\pi=(\pi_t)_{t\in \mathbb{N}}$ is a sequence of deterministic functions, where $\pi_{t}(\hist)$ specifies a probability distribution over the action set $\A$. Let $\D(\A)$ denote the set of probability distributions over $\A$. The action $A_t$ is a selected by sampling independently from $\pi_t(\hist)$. With some abuse of notation, we will typically write this distribution as $\pi_{t}$, where $\pi_{t}(a) = \Prob( A_t = a| \hist)$ denotes the probability assigned to action $a$ given the observed history.  We explicitly display the dependence of regret on the policy $\pi$, letting $\E\left[{\rm Regret}(T, \pi)\right]$ denote the expected value given by \eqref{eq: expected regret} when the actions $(A_{1},..,A_{T})$ are chosen according to $\pi$.

\paragraph{Further notation.}

 Set $\alpha_t(a) = \Prob\left( A^*=a \vert \hist  \right)$ to be the posterior distribution of $A^*$. For two probability measures $P$ and $Q$ over a common measurable space, if $P$ is absolutely continuous with respect to $Q$, the {\it Kullback-Leibler divergence} between $P$ and $Q$ is
\begin{equation}
D_{\rm KL}(P || Q)= \intop \log \left( \frac{dP}{dQ} \right)dP
\end{equation}
where $\frac{dP}{dQ}$ is the Radon--Nikodym derivative of $P$ with respect to $Q$. For a probability distribution $P$ over a finite set $\mathcal{X}$, the {\it Shannon entropy} of $P$ is defined as $H(P) =- \sum_{x\in\mathcal{X}} P(x) \log\left( P(x) \right)$.
The {\it mutual information} under the posterior distribution between two random variables $X_1:\Omega \rightarrow \mathcal{X}_{1}$, and $X_2:\Omega \rightarrow \mathcal{X}_{2}$, denoted by
\begin{equation}
I_{t}(X_1 ; X_2) := \DKL\left( \mathbb{P}\left( (X_{1}, X_{2})  \in \cdot \vert \hist \right) \,\, || \,\, \mathbb{P}\left( X_{1}  \in \cdot \vert \hist \right) \mathbb{P}\left( X_{2}  \in \cdot \vert \hist \right)  \right),
\end{equation}
is the Kullback-Leibler divergence between the joint posterior distribution of $X_{1}$ and $X_{2}$ and the product of the marginal distributions. Note that $I_{t}(X_1 ; X_2)$ is a random variable because of its dependence on the conditional probability measure $\Prob\left( \cdot \vert \hist \right)$.

To reduce notation, we define the {\it information gain} from an action $a$ to be $g_{t}(a) := I_{t}(A^* ; Y_{t,a})$. As shown for example in Lemma 5.5.6 of \citet{gray2011entropy}, this is equal to the expected reduction in entropy of the posterior distribution of $A^*$ due to observing $Y_{t}(a)$:
\begin{equation}
g_{t}(a) = \E \left[H(\alpha_{t}) - H(\alpha_{t+1}) \vert \hist, A_{t}=a\right],
\end{equation}
which plays a crucial role in our results. Let $\Delta_{t}(a):= \E \left[R_{t,A^*}-R_{t,a}  \vert \hist \right]$ denote the expected instantaneous regret of action $a$ at time $t$.

We use overloaded notation for $g_t(\cdot)$ and $\Delta_{t}(\cdot)$. For an action sampling distribution $\pi \in \D(\A)$, $ g_{t}(\pi) := \sum_{a\in \A} \pi(a) g_{t}(a)$ denotes the expected information gain when actions are selected according to $\pi$, and $\Delta_{t}(\pi) = \sum_{a\in\A} \pi(a)\Delta_{t}(a)$ is defined analogously. Finally, we sometimes use the shorthand notation $\E_{t}[\cdot]=\E_{t}[\cdot | \hist]$ for conditional expectations under the posterior distribution, and similarly write $\Prob_{t}(\cdot)=\Prob(\cdot | \hist)$.

\section{Algorithm design principles}

The primary contribution of this paper is information-directed sampling (IDS), a general principle for designing action-selection algorithms.
We will define IDS in this section, after discussing motivations underlying its structure.  Further, through a set of examples, we will illustrate
how alternative design principles fail to account for particular kinds of information and therefore can be dramatically outperformed by IDS.

\subsection{Motivation}
\label{se:motivation}

Our goal is to minimize expected regret over a time horizon $T$.  This is achieved by a \emph{Bayes-optimal} policy, which, in principle, can be computed via dynamic programming.
Unfortunately, computing, or even storing, this Bayes-optimal policy is generally infeasible.  For this reason, there has been significant interest in developing computationally efficient heuristics.

As with much of the literature, we are motivated by contexts where the time horizon $T$ is ``large.''  For large $T$ and moderate times $t \ll T$,
the mapping from belief state to action prescribed by the Bayes-optimal policy does not vary significantly from one time period to the next.  As such,
it is reasonable to restrict attention to stationary heuristic policies.  IDS falls in this category.

IDS is motivated largely by a desire to overcome shortcomings of currently popular design principles.  In particular, it accounts for 
kinds of information that alternatives fail to adequately address:
\begin{enumerate}
\item {\bf Indirect information.}  IDS can select an action to obtain useful feedback about other actions even if there will be no useful feedback about the selected action.
\item {\bf Cumulating information.}  IDS can select an action to obtain feedback that does not immediately enable higher expected reward
but can eventually do so when combined with feedback from subsequent actions.
\item {\bf Irrelevant information.}  IDS avoids investments in acquiring information that will not help to determine which actions ought to be selected.
\end{enumerate}
Examples presented in Section \ref{sebsec: alternatives} aim to contrast the manner in which IDS and alternative approaches treat these kinds of information.

It is worth noting that we refer to IDS as a {\it design principle} rather than an {\it algorithm}.  The reason is that IDS does not specify steps to be carried out in terms
of basic computational operations but only an abstract objective to be optimized.  As we will discuss later, for many problem classes of interest, like the Bernoulli bandit, the Gaussian
bandit, and the linear bandit, one can develop tractable algorithms that implement IDS.  The situation is similar for upper confidence bounds,
Thompson sampling, expected improvement maximization, and knowledge gradient; these are abstract design principles that lead to
tractable algorithms for specific problem classes.

\subsection{Information-directed sampling}

IDS balances between obtaining low expected regret in the current period and acquiring new information about which action is optimal. It does this by minimizing over all action sampling distributions $\pi \in \D(\A)$ the ratio between the square of expected regret $\Delta_{t}(\pi)^2$ and information gain $g_{t}(\pi)$ about the optimal action $A^*$. In particular, the policy  $\pi^{{\rm IDS}}=\left(\pi_{1}^{{\rm IDS}}, \pi_{2}^{{\rm IDS}},\ldots \right)$  is defined by:
\begin{equation}\label{eq: definition of ratio policy} \pi_{t}^{{\rm IDS}} \in \underset{\pi \in \D(\A)}{\arg\min} \left\{ \Psi_{t}\left( \pi \right):= \frac{\Delta_{t}(\pi)^2}{g_{t}(\pi)} \right\}. \end{equation}
We call $\Psi_{t}(\pi)$ the {\it information ratio} of an action sampling distribution $\pi$. It measures the squared regret incurred per-bit of information acquired about the optimum. IDS myopically minimizes this notion of cost-per-bit of information in each period.

Note that IDS is {\it stationary randomized policy}: randomized in that each action is randomly sampled from a distribution and stationary in that this action distribution is determined by the posterior distribution of $\theta$ and otherwise independent of the time period.  It is natural to wonder whether randomization plays a fundamental role or if a stationary deterministic policy can offer similar behavior.  The following example sheds light on this matter.
\begin{example}[{\bf A known standard}]\label{ex: known standard}
Consider a problem with two actions $\mathcal{A} = \left\{a_1, a_2 \right\}$. Rewards from $a_1$ are known to be distributed Bernoulli(1/2).  The distribution of rewards from $a_2$ is Bernoulli(3/4) with prior probability $p_0$ and is Bernoulli(1/4) with prior probability $1-p_0$.
\end{example}
Consider a stationary deterministic policy for this problem.  With such a policy, each action $A_t$ is a deterministic function of $p_{t-1}$, the posterior probability conditioned on observations made through period $t-1$.  Suppose that for some $p_0 > 0$, the policy selects $A_1 = a_1$.  Since this is an uninformative action, $p_t = p_0$ and $A_t = a_1$ for all $t$, and thus, expected regret grows linearly with time.  If, on the other hand, $A_1 = a_2$ for all $p_0 > 0$ then $A_t = a_2$ for all $t$, which again results in expected regret that grows linearly with time.  It follows that, for any deterministic stationary policy, there exists a prior probability $p_0$ such that expected regret grows linearly with time.

In Section \ref{se:bounds}, we will establish a sub-linear bound on expected regret of IDS.  The result implies that, when applied to the preceding example, the expected regret of IDS does not grow linearly as does that of any stationary deterministic policy.  This suggests that randomization plays a fundamental role.

It may appear that the need for randomization introduces great complexity since the solution of the optimization problem (\ref{eq: definition of ratio policy}) 
takes the form of a distribution over actions.  However, an important property of this problem dramatically simplifies solutions.  In particular, as we will establish in
Section \ref{se:computation}, there is always a distribution with support of at most two actions that attains the minimum in \eqref{eq: definition of ratio policy}.

\subsection{Alternative design principles}
\label{sebsec: alternatives}

Several alternative design principles have figured prominently in the literature.  However, each of them fails to adequately
address one or more of the categories of information enumerated in Section \ref{se:motivation}.
This motivated our development of IDS.  In this section, we will illustrate through a set of examples how IDS accounts for such information
while alternatives fail.

\subsubsection{Upper confidence bounds and Thompson sampling}\label{subsubsec: UCB and TS}

Upper confidence bound (UCB) and Thompson sampling (TS) are two of the most popular principles for balancing between exploration and exploitation. As data is collected, both approaches do not only estimate the rewards generated by different actions, but carefully track the uncertainty in their estimates. They then continue to experiment with all actions that could {\it plausibly} be optimal given the observed data. This guarantees actions are not prematurely discarded, but, in contrast to more naive approaches likes the $\epsilon$-greedy algorithm, also ensures that samples are not wasted on clearly suboptimal actions.

With a UCB algorithm, actions are selected through two steps. First, for each action $a \in \A$ an upper confidence bound $B_{t}(a)$ is constructed. Then, the algorithm selects an action $A_{t} \in \arg\max_{a \in \A} B_{t}(a)$ with maximal upper confidence bound. The upper confidence bound $B_{t}(a)$ represents the greatest mean reward value that is statistically plausible. In particular, $B_{t}(a)$ is typically constructed to be optimistic $(B_{t}(a) \geq \E[ R_{t,a} | \theta])$ and asymptotically consistent ($B_t(a) \rightarrow \E\left[ R_{t,a} | \theta \right]$ as data about the action $a$ accumulates).

A TS algorithm simply samples each actions according to the posterior probability that it is optimal. In particular, at each time $t$, an action is sampled from $\pi^{\rm TS}_{t} = \alpha_t$.  This means that for each $a \in \A$,  $\Prob (A_{t}=a | \hist)=\Prob (A^*=a | \hist)=\alpha_{t}(a)$. This algorithm is sometimes called {\it probability matching} because the action selection distribution is {\it matched} to the posterior distribution of the optimal action.

For some problem classes, UCB and TS lead to efficient and empirically effective algorithms with strong theoretical guarantees. Specific UCB and TS algorithms are known to be asymptotically efficient for multi-armed bandit problems with independent arms \cite{lai1985asymptotically, lai1987adaptive, KL-UCB2013, kaufmann2012thompson, agrawal2013further} and satisfy strong regret bounds for some problems with dependent arms \cite{dani2008stochastic, rusmevichientong2010linearly, srinivas2012information,  bubeck2011xarmed,  filippi2010parametric, gopalan2014thompson, russo2014learning}.

Unfortunately, as the following examples demonstrate, UCB and TS do not pursue indirect information and because of that can perform very poorly relative to IDS for some natural problem classes.  A common feature of UCB and TS that leads to poor performance in these examples is that they restrict attention to sampling actions that have some chance of being optimal.  This is the case with TS because each action is selected according to the probability that it is optimal.  With UCB, the upper-confidence-bound of an action known to be suboptimal will always be dominated by others.

Our first example is somewhat contrived but designed to make the point transparent.
\begin{example}\label{ex: revealing action}
{\bf (a revealing action)}
Let $\A = \{0,1,\ldots,K\}$ consist of $K+1$ actions and suppose that $\theta$ is drawn uniformly at random from a finite set $\Theta=\{1,\ldots,K \}$ of $K$ possible values. Consider a problem with bandit-feedback $Y_{t,a}=R_{t,a}$. Under $\theta$, the reward of action $a$ is $$R_{t,a}=\begin{cases}
1 & \theta=a\\
0 & \theta\ne a,\, a\ne0\\
\frac{1}{2\theta} & a=0.
\end{cases}$$
\end{example}
Note that action $0$ is known to never yield the maximal reward, and is therefore never selected by TS or UCB. Instead, these algorithms will select among actions $\{1,\ldots,K\}$, ruling out only a single action at a time until a reward 1 is earned and the optimal action is identified. Their expected regret therefore grows linearly in $K$. IDS is able to recognize that much more is learned by drawing action $0$ than by selecting one of the other actions. In fact, selecting action $0$ immediately identifies the optimal action. IDS selects this action, learns which action is optimal, and selects that action in all future periods.  Its regret is independent of $K$.

Our second example may be of greater practical significance.  It represents the simplest case of a sparse linear model.
\begin{example}
\label{ex:sparse}
{\bf (sparse linear model)}
Consider a linear bandit problem where $\A \subset \mathbb{R}^d$ and the reward from an action $a \in \A$ is $a^T \theta^*$. The true parameter $\theta$ is known to be drawn uniformly at random from the set of 1-sparse vectors $\Theta=\{ \theta' \in \{0,1 \}^d : \| \theta' \|_0 =1 \}$. For simplicity, assume $d=2^m$ for some $m\in \mathbb{N}$. The action set is taken to be the set of vectors in $\{ 0,1 \}^d$ normalized to be a unit vector in the $L^1$ norm:
 $\A = \left\{ \frac{x}{\| x\|_1} : x\in \{0,1  \}^d, x\neq 0 \right\}$. 
\end{example} 
For this problem, when an action $a$ is selected and $y=a^T \theta \in \{0, 1/ \| a\|_{0} \}$ is observed, each $\theta' \in \Theta$ with $a^T \theta' \neq y$ is ruled out.   Let $\Theta_{t}$ denote the set of all parameters in $\Theta$ that are consistent with the rewards observed up to time $t$ and let $\mathcal{I}_{t} = \{i \in \{1,\ldots,d\} : \theta'_{i} =1, \theta' \in \Theta_{t} \}$ denote the corresponding set of possible positive components.

Note that $A^* = \theta$. That is, if $\theta$ were known, choosing the action $\theta$ would yield the highest possible reward. TS and UCB algorithms only choose actions from the support of $A^*$ and therefore will only sample actions $a \in \A$ that, like $A^*$, have only a single positive component. Unless that is also the positive component of $\theta$, the algorithm will observe a reward of zero and rule out only one element of $\mathcal{I}_t$.  In the worst case, the algorithm requires $d$ samples to identify the optimal action.

Now, consider an application of IDS to this problem. The algorithm essentially performs binary search: it selects $a\in \A$ with $a_{i}>0$ for half of the components $i \in \mathcal{I}_{t}$ and $a_{i}=0$ for the other half as well as for any $i \notin \mathcal{I}_{t}$. After just $\log_{2}(d)$ time steps the true value of the parameter vector $\theta$ is identified.

To see why this is the case, first note that all parameters in $\Theta_{t}$ are equally likely and hence the expected reward of an action $a$ is $\frac{1}{|\mathcal{I}_{t}|} \sum_{i \in I_{t}} a_i$. Since $a_i \geq 0$ and $\sum_{i} a_i=1$ for each $a\in \A$, every action whose positive components are in $I_{t}$ yields the highest possible expected reward of $1/| \mathcal{I}_{t}|$. Therefore, binary search minimizes expected regret in period $t$ for this problem. At the same time, binary search is assured to rule out half of the parameter vectors in $\Theta_{t}$ at each time $t$. This is the largest possible expected reduction, and also leads to the largest possible information gain about $A^*$. Since binary search both minimizes expected regret in period $t$ and uniquely maximizes expected information gain in period $t$, it is the sampling strategy followed by IDS.

In this setting we can explicitly calculate the information ratio of each policy, and the difference between them highlights the advantages of information-directed sampling. We have
\begin{eqnarray*}
\Psi_{1}(\pi^{\rm TS}_1)= \frac{(d-1)^2/d^2}{\frac{\log(d)}{d} + \frac{d-1}{d}\log\left( \frac{d}{d-1} \right)}
\sim \frac{d}{\log(d)} & \qquad &  \Psi_{1}(\pi^{{\rm IDS}}_1)= \frac{1}{\log(2)}\left(1 - \frac{1}{d}\right)^2
\sim   \frac{1}{\log(2)}
\end{eqnarray*}
where $h(d) \sim f(d)$ if $h(d)/f(d) \rightarrow 1$ as $d \rightarrow \infty$. When the dimension $d$ is large, $\Psi_{1}(\pi^{{\rm IDS}}_{1})$ is much smaller.

Our final example involves an assortment optimization problem.
\begin{example}
{\bf (assortment optimization)}
\label{ex:assortment}
Consider the problem of repeatedly recommending an assortment of products to a customer. The customer has unknown type $\theta \in \Theta$ where $|\Theta|=n$. Each product is geared toward customers of a particular type, and the assortment $a \in \A = \Theta^m$ of $m$ products offered is characterized by the vector of product types $a=(a_1,\ldots,a_m)$. We model customer responses through a random utility model in which customers are more likely to derive high value from a product geared toward their type.  When offered an assortment of products $a$, the customer associates with the $i$th product utility $U_{\theta, i}^{(t)}(a)= \beta \mathbf{1}_{\{a_i =\theta\}}+W_{i}^{(t)}$, where $W_{i}^{t}$ follows an extreme--value distribution and $\beta \in \mathbb{R}$ is a known constant. This is a standard multinomial logit discrete choice model.  The probability a customer of type $\theta$ chooses product $i$ is given by
$$ \frac{\exp\{  \beta \mathbf{1}_{\{a_i =\theta \}} \}}{\sum_{j=1}^m \exp\{ \beta \mathbf{1}_{\{a_j =\theta \}}\}}. $$
When an assortment $a$ is offered at time $t$, the customer makes a choice $I_{t}=\arg \max _{i} U_{\theta i}^{(t)}(a)$ and leaves a review $U_{\theta I_{t}}^{(t)}(a)$ indicating the utility derived from the product, both of which are observed by the recommendation system. The reward to the recommendation system is the normalized utility of the customer $U_{\theta I_{t}}^{(t)}(a) / \beta$.
\end{example}
If the type $\theta$ of the customer were known, then the optimal recommendation would be $A^*=(\theta, \theta, \ldots,\theta)$, which consists only of products targeted at the customer's type. Therefore, both TS and UCB would only offer assortments consisting of a single type of product. Because of this, TS and UCB each require order $n$ samples to learn the customer's true type. IDS will instead offer a {\it diverse} assortment of products to the customer, allowing it to learn much more quickly.

To render issues more transparent, suppose that $\theta$ is drawn uniformly at random from $\Theta$ and consider the behavior of each type of algorithm in the limiting case where $\beta \rightarrow \infty$. In this regime, the probability a customer  chooses a product of type $\theta$ if it is available tends to 1, and the normalized review $\beta^{-1} U_{\theta I_{t}}^{(t)}(a)$ tends to  $\mathbf{1}_{\{a_{I_{t}}=\theta\}}$, an indicator for whether the chosen product is of type $\theta$. The initial assortment offered by IDS will consist of $m$ different and previously untested product types. Such an assortment maximizes both the algorithm's expected reward in the next period and the algorithm's information gain, since it has the highest probability of containing a product of type $\theta$. The customer's response almost perfectly indicates whether one of those items was of type $\theta$. The algorithm continues offering assortments containing $m$ unique, untested, product types until a review near $U_{\theta I_{t}}^{(t)}(a)\approx 1$ is received. With extremely high probability, this takes at most $\left \lceil{n/m}\right \rceil$ time periods. By diversifying the $m$ products in the assortment, the algorithm learns a factor of $m$ times faster.

As in the previous example, we can explicitly calculate the information ratio of each policy, and the difference between them highlights the advantages of IDS.  The information ratio of IDS is more than $m$ times smaller:
\begin{eqnarray*}
\Psi_{1}(\pi^{\rm TS}_1)= \frac{\left( 1-\frac{1}{n}\right)^2}{\frac{1}{n} \log(n) + \frac{n-1}{n}\log\left( \frac{n}{n-1} \right)}
\sim \frac{n}{\log(n)} &\hspace{10pt} &  \Psi_{1}(\pi^{{\rm IDS}}_1)= \frac{\left(1-\frac{m}{n}\right)^2}{\frac{m}{n}\log(\frac{n}{m})  + \frac{n-m}{n}\log\left( \frac{n}{n-m} \right)} \leq  \frac{n}{m \log(\frac{n}{m})}.
\end{eqnarray*}

\subsubsection{Other information-directed approaches}
\label{subsec: other info directed}

Another natural information-directed algorithm aims to maximize the information acquired about the uncertain model parameter $\theta$. In particular, consider an algorithm that selects the action at time $t$ that maximizes the weighted combination of the expected reward the action generates and the information it generates about the uncertain model parameter $\theta$:
$
  \E_{t}[R_{t,a}] + \lambda I_{t}(Y_{t,a} ; \theta ).
$
Throughout this section we will refer to this algorithm as $\theta$-IDS. While such an algorithm can perform well on particular examples, the next example highlights that it may invest in acquiring information about $\theta$ that is irrelevant to the decision problem.


\begin{example}{\bf(unconstrained assortment optimization)}
Consider again the problem of repeatedly recommending assortments of products to a customer with unknown preferences. The recommendation system can choose any subset of products $a\subset \{1,..,n\}$ to display. When offered assortment $a$ at time $t$, the customer chooses the item $J_{t} = \arg\max_{i \in a} \theta_i$ and leaves the review $R_{t,a}= \theta_{J_t}$ where $\theta_i$ is the utility associated with product $i$. The recommendation system observes both $J_t$ and the review $R_{t,a}$, and has the goal of learning to offer the assortment that yields the best outcome for the customer and maximizes the review $R_{t,a}$. Suppose that $\theta$ is drawn as a uniformly random permutation of the vector $(1, \frac{1}{2}, \frac{1}{3},\ldots,\frac{1}{n})$. The customer is known to assign utility 1 to her most preferred item, 1/2 to the next best item, 1/3 to the third best, and so on, but the rank ordering is unknown.
\end{example}
In this example, no learning is required to offer an optimal assortment: since there is no constraint on the size of the assortment, it's always best to offer the full collection of products $a=\{1,\ldots,n\}$ and allow the customer to choose the most preferred. Offering this assortment reveals which item is most preferred by the customer, but it reveals nothing about her preferences about others. When applied to this problem, $\theta$-IDS begins by offering the full assortment $A_1 = \{1,\ldots,n\}$, which yields a reward of 1, and, by revealing the top item, yields information of $I_{1}(Y_{1,A_1} ; \theta) = \log(n)$. But, if $1/2< \lambda \log(n-1)$, which is guaranteed for sufficiently large $n$, it  continues experimenting with suboptimal assortments. In the second period, it will offer the assortment $A_2$ consisting of all products except $\arg\max_{i} \theta_i$. Playing this assortment reveals the customer's second most preferred item, and yields information gain $I_{2}(Y_{2,A_2} ; \theta)= \log(n-1)$. This process continues until the first period $k$ where $\lambda \log(n+1-k) < 1-1/k$.

In order to learn to offer effective assortments, $\theta$-IDS tries to learn as much as possible about the customer's preferences. In doing this, the algorithm inadvertently invests experimentation effort in information that is irrelevant to choosing an optimal assortment. On the other hand, IDS recognizes that the optimal assortment $A^*=\{1,\ldots,n\}$ does not depend on full knowledge of the vector $\theta$, and therefore does not invest in identifying $\theta$.

As shown in Section \ref{subsec: theta-IDS}, our analysis can be adapted to provide regret bounds in  for a version of IDS that uses information gain with respect to $\theta$, rather than with respect to $A^*$. These regret bounds depends on the entropy of $\theta$, whereas the bound for IDS depends on the entropy of $A^*$, which can be much smaller.

\subsubsection{Expected improvement and the knowledge gradient}\label{subsec: expected improvement}

We now consider two algorithms which measure the quality of the best decision that can be made based on current information, and encourage gathering observations that are expected to immediately increase this measure. The first is the expected improvement algorithm, which is one of the most widely used techniques in the active field of Bayesian optimization (see \cite{brochu2009tutorial}). Define $\mu_{t,a} = \E[R_{t,a} | \hist]$ to be the expected reward generated by $a$ under the posterior,  and $V_t = \max_{a'} \mu_{t,a'}$ to be the best objective value attainable given current information.  The expected improvement of action $a$ is defined to be $\E[ \max\{f_{\theta}(a), V_t\}| \hist]$, where $f_{\theta}(a) = \E[R_{t,a}|\theta]$ is the expected reward generated by action $a$ under the unknown true parameter $\theta$.
The EGO algorithm aims to identify high performing actions by sequentially sampling those  that yield the highest expected improvement. Similar to UCB algorithms, this encourages the selection of actions that could potentially offer great performance. Unfortunately, like these UCB algorithms, this measure of improvement does not place value on indirect information: it won't select an action that provides useful feedback about other actions unless the mean-reward of that action might exceed $V_t$. For example, the expected improvement algorithm cannot treat the problem described in Example \ref{ex: revealing action} in a satisfactory manner.

The knowledge gradient algorithm \cite{ryzhov2012knowledge} uses a modified improvement measure. At time $t$, it computes
\[
v^{KG}_{t,a}  :=   \E \left[ V_{t+1}   \vert \hist, A_{t}=a \right] -  V_{t}
\]
for each action $a$. If $V_{t}$ measures the quality of the decision that can be made based on current information, then $v^{KG}_{t,a}$ captures the immediate  improvement in decision quality due to sampling action $a$ and observing $Y_{t,a}$. For a problem with time horizon $T$, the knowledge gradient (KG) policy selects an action in time period $t$ by maximizing  $\mu_{t,a}+ (T-t) v^{KG}_{t,a}$ over actions $a\in \A$.

Unlike expected-improvement, the measure $v^{KG}_{t,a}$ of the value of sampling an action places value on indirect information. In particular, even if an action is known to yield low expected reward, sampling that action could lead to a significant increase in $V_t$ by providing information about other actions. Unfortunately, there are no general guarantees for KG, and it sometimes struggles with cumulating information; individual observations that provide information about $A^*$ may not be immediately useful for making decisions in the next period, and therefore may lead to no improvement in $V_t$. 

Example \ref{ex: known standard} provides one simple illustration of this phenomenon. In that example, action 1 is known to yield a mean reward of 1/2. When $p_0 \leq 1/4$, upon sampling action 2 and observing a reward 1, the posterior expected reward of action 2 becomes 
$$ \E[ R_{2,a_2} |  R_{1, a_2}=1] = \frac{p_0 (3/4)}{p_0 (3/4)+(1-p_0)(1/4)}  \leq 1/2.$$ 
In particular, a single sample could never be influential enough to change which action has the highest posterior expected reward.   Therefore, $v^{KG}_{t,a_2}=0$, and the KG decision rule selects action 1 in the first period. Since nothing is learned from the resulting observation, it will continue selecting action 1 in all subsequent periods. Even as the time horizon $T$ tends to infinity, the KG policy would never select action 2. Its cumulative regret over $T$ time periods is  equal to $ (p_0/4)T$, which grows linearly with $T$.


In this example, although sampling action 2 will not immediately shift the decision-maker's prediction of the best action ($\arg\max_{a} \E[\theta_a | \mathcal{F}_1]$), these samples  influences her posterior beliefs and reduce uncertainty about which action is optimal. As a result, IDS will always assign positive probability to sampling the second action. More broadly, IDS places value on information that is pertinent to the decision problem, even if that information won't directly improve performance on its own. This is useful when one must combine multiple pieces of information in order to effectively learn.

To address problems like Example~\ref{ex: known standard}, \citet{frazier2010paradoxes} propose KG* -- a modified form of KG that considers the value of sampling a single action many times. This helps to address some cases -- those where a single sample of an action provides no value even though sampling the action several times could be quite valuable -- but this modification may not adequately assess information gain in  more general problems. KG* may explore very inefficiently, for instance, in the sparse linear bandit problem considered in Example~\ref{ex:sparse}. However, the performance of KG* on this problem depends critically on how ties are broken among actions maximizing $\mu_{t,a}+(T-t)v_{t,a}^{\rm KG}$. To avoid such ambiguity, we instead consider the following modification to Example~\ref{ex:sparse}.

\begin{example} \label{ex:sparse with outside option}
	{\bf (sparse linear model with an outside option)}
	Consider a modification to Example \ref{ex:sparse} in which the agent has access to a outside option that generates known, but suboptimal, rewards. The action set is $\A= \{O\} \cup \A' $, where the outside option $O$ is know to always yield reward 1/2, and $\A'=\left\{ \frac{x}{\| x\|_1} : x\in \{0,1  \}^d, x\neq 0 \right\}$ is the action space considered in Example \ref{ex:sparse}. As before, the reward generated by action $a\in \A'$ is $a^T \theta$ where $\theta$ is drawn uniformly at random from the set of 1-sparse vectors $\Theta=\{ \theta' \in \{0,1 \}^d : \| \theta' \|_0 =1 \}$. 
	 As before, assume for simplicity that $d=2^m$ for some $m\in \mathbb{N}$. 
\end{example}
When the horizon $T$ is long, the inclusion of the outside option should be irrelevant. Indeed, nothing is learned by sampling the outside option, and it is known apriori to be suboptimal: it generates a reward of 1/2 whereas the optimal action generates a reward of $1=\max_{a\in \A'} a^T \theta$. An optimal algorithm for Example~\ref{ex:sparse with outside option} should sample actions in $\A'$ until the optimal action is identified and should do so in a manner that minimizes the expected regret incurred. 

Note that in the absence of observation noise, KG and KG* are equivalent, and so we will not distinguish between these two algorithms. We will see that when $T$ is large,  KG always samples actions in $\A'$ until the optimal action is identified. However, the  expected number of samples required, and the expected regret incurred, both scale linearly with the dimension $d$. IDS, on the other hand, identifies the optimal action after only $\log_{2}(d)$ steps. 

To see why, note that unless the agent has exactly identified the parameter $\theta$, selecting the outside option will always generate the highest expected reward in the next period. Only selecting an action with a single positive component could immediately reveal $\theta$, so in response KG only places value on the information generated by such actions. When the horizon is large, KG prioritizes such information over the safe reward offered by the outside option. It therefore engages in exhaustive search, sequentially checking each component $i\in \{1,\ldots, d\}$ of $\theta$ until $\theta_i=1$ is observed. 


To show this more precisely, let us reintroduce some notation used in Example \ref{ex:sparse}. As before, when an action $a \in \A'$ is selected and $y=a^T \theta \in \{0, 1/ \| a\|_{0} \}$ is observed, each $\theta' \in \Theta$ with $a^T \theta' \neq y$ is ruled out.   Let $\Theta_{t}$ denote the set of all parameters in $\Theta$ that are consistent with the rewards observed up to time $t$ and let $\mathcal{I}_{t} = \{i \in \{1,\ldots,d\} : \theta'_{i} =1, \theta' \in \Theta_{t} \}$ denote the corresponding set of possible positive components.

The best reward value attainable given current information is
\[
V_{t} = \max \left\{ 1/2 ,\, \max_{a\in \A'} \E[ \theta^T a | \hist ] \right\} = \max\left\{ 1/2 ,\, \frac{1}{|\Theta_{t}|} \right\}.
\]
Therefore, $V_{t} = 1/2$ whenever $|\Theta_{t}|>1$. For remaining indices $i \in \mathcal{I}_{t}$, selecting the standard basis vector $e_i$ and observing $\theta^T e_i =1$ establishes that  $\theta =e_i$. As a result, $v_{t, e_i}^{\rm KG} = (1/2)\Prob(e_i^T \theta = 1 | \hist) = \frac{1}{2 |\Theta_t|}>0$. For other actions, observing $a^T \theta \in \{0,1/\|a\|_0\}$ is not sufficient to identify $\theta$, and so $v_{t, a}^{\rm KG} =0$. When the horizon $T$ is large, KG always select one of the actions with $v_{t, a}^{\rm KG}>0$, and so the algorithm selects only standard basis vectors until $\theta$ is revealed  

IDS, on the other hand, essentially performs binary search. In the first period, it selects some permutation of the action $a=(2/d,...,2/d, 0,...,0)$. By observing either $a^T \theta=0$ or $a^T \theta =1$, it rules out half the parameter vectors in $\Theta_1$. Continuing this binary search process, the parameter is identified using $O(\log_{2}(d))$ steps.  

To see why this is the case, let us focus on the first period. As in Example \ref{ex:sparse}, selecting an action with $d/2$ positive components, like $a=(2/d,...,2/d, 0,...,0)$, offers strictly maximal 
information gain $I_{1}(A^* ; a^T \theta ) = \log(2)$ and weakly maximal reward $\E[a^T \theta]= 1/d$ among all actions $a \in \A'$. As a result, any action in $\A'$ that is not a permutation of $a$ is strictly dominated and is never sampled by IDS. It turns out that IDS also has zero probability of selecting the outside option in this case. Indeed, IDS selects $O$ with probability $1-\alpha^*$, where $\alpha^*$ can be attained as the minimizer of the information ratio
\[ 
\alpha^* = \underset{\alpha \in [0,1]}{\arg\min} \frac{((1-\alpha)/2 +\alpha(1 - 1/d) )^2}{\alpha \log_2(d)} = \underset{\alpha \in [0,1]}{\arg\min}\,\, \frac{1}{2 \sqrt{\alpha}} + \sqrt{\alpha}\left(\frac{1}{2} - \frac{1}{d} \right) = 1.
\]
This process continues inductively. Another step of binary search again rules out half the components in $|\Theta_1|$, leading to an information gain of $\log(2)$ bits. The regret incurred is even lower--now only $(1-2/d)$
--and hence the algorithm again has zero probability of selecting the outside option. Iterating this process, the true positive component of $\theta$ is identified after $\log_{2}(d)$ steps.

\section{Regret bounds}
\label{se:bounds}

This section establishes regret bounds for information-directed sampling for several of the most widely studied classes of online optimization problems. These regret bounds follow from our recent information theoretic-analysis of Thompson sampling \citep{russo2016info}. In the next subsection, we establish a regret bound for any policy in terms of its information ratio. Because the information-ratio of IDS is always smaller than that of TS, the bounds on the information ratio of TS provided in \citet{russo2016info} immediately yield regret bounds for IDS for a number of important problem classes.  

\subsection{General bound}

We begin with a general result that bounds the regret of \emph{any policy} in terms of its information ratio and the entropy of the optimal action distribution. Recall that we have defined the information ratio of an action sampling distribution to be $\Psi_{t}( \pi ) : = \Delta_{t}(\pi)^2 / g_{t}(\pi)$; it is the  squared expected regret the algorithm incurs per-bit of information it acquires about the optimum. The entropy of the optimal action distribution $H(\alpha_1)$ captures the magnitude of the decision-maker's initial uncertainty about which action is optimal. One can then interpret the next result as a bound on regret that depends on the cost of acquiring new information and the total amount of information that needs to be acquired.
\begin{prop}\label{prop: regret bound average information ratio} For any policy $\pi=(\pi_1, \pi_2, \pi_3, \ldots)$ and time $T\in \mathbb{N}$,
$$ \E \left[\mathrm{Regret}\left(T, \pi \right) \right] \leq \sqrt{ \overline{\Psi}_{T}(\pi) H(\alpha_{1}) T }.$$
where
\[
\overline{\Psi}_{T}(\pi) \equiv \frac{1}{T} \sum_{t=1}^{T} \E_{\pi} [  \Psi_{t}(\pi_t ) ]
\]
is the average expected information ratio under $\pi$.
\end{prop}
We will use the following immediate corollary of Proposition \ref{prop: regret bound average information ratio}, which relies on a uniform bound on the information ratio of the form $\Psi_{t}(\pi_t) \leq \lambda$ rather than a bound on the average expected information ratio.
\begin{cor}\label{cor: general regret bound}
Fix a deterministic $\lambda \in \mathbb{R}$ and a policy $\pi=\left(\pi_{1}, \pi_{2}, \ldots \right)$ such that $\Psi_{t}(\pi_{t}) \leq \lambda$ almost surely for each $t\in \{1,..,T\}$. Then,
$$ \E \left[\mathrm{Regret}\left(T, \pi \right) \right] \leq \sqrt{\lambda H(\alpha_{1}) T }.$$
\end{cor}

\subsection{Specialized bounds on the minimal information ratio}

We now establish upper bounds on the information ratio of IDS in several important settings, which yields explicit regret bounds when combined with Corollary \ref{cor: general regret bound}. These bounds show that, in any period, the algorithm's expected regret can only be large if it is expected to acquire a lot of information about which action is optimal. In this sense, it effectively balances between exploration and exploitation in {\it every} period. For each problem setting, we will compare our upper bounds on expected regret with known lower bounds.

The bounds on the information ratio also help to clarify the role it plays in our results: it roughly captures the extent to which sampling some actions allows the decision maker to make inferences about {\it other} actions. In the worst case, the ratio depends on the number of actions, reflecting the fact that actions could provide no information about others. For problems with full information, the information ratio is bounded by a numerical constant, reflecting that sampling one action perfectly reveals the rewards that would have been earned by selecting any other action. The problems of online linear optimization under ``bandit feedback'' and under ``semi-bandit feedback'' lie between these two extremes, and the ratio provides a natural measure of  each problem's information structure.  In each case, our bounds reflect that IDS is able to automatically exploit this structure. 

The proofs of these bounds follow from our recent analysis of Thompson sampling, and the implied regret bounds are the same as those established for Thompson sampling.  In particular, since $\Psi_{t}(\pi_{t}^{\rm IDS}) \leq \Psi_{t}( \pi_{t}^{\rm TS})$ where $\pi^{\rm TS}$ is the Thompson sampling policy, it is enough to bound $\Psi_{t}( \pi_t^{\rm TS})$. Several bounds on the information-ratio of TS were provided by \citet{russo2016info}, and we defer to that paper for the proofs. While the analysis is similar in the cases considered here,  IDS outperforms Thompson sampling in simulation, and, as we  highlighted in the previous section, is sometimes provably much more informationally efficient. 

In addition to the bounds stated here, recent work by \citet{bubeck2015bandit} and \citet{bubeck2016multi} bounds the information ratio when the reward function is convex, and uses this to study the order of regret in adversarial bandit convex optimization. This points to a broader potential of using information-ratio analysis to study the information-complexity of general online optimization problems.  


To simplify the exposition, our results are stated under the assumption that rewards are uniformly bounded. This effectively controls the worst-case variance of the reward distribution, and as shown in the appendix of \citet{russo2016info}, our results can be extended to the case where reward distributions are sub-Gaussian.
\begin{assumption}\label{assum: bounded rewards}
$\underset{\overline{y}\in \mathcal{Y}}{\sup} R(\overline{y})-\underset{\underline{y}\in \mathcal{Y}}{\inf}R(\underline{y}) \leq 1.$
\end{assumption}
\subsubsection{Worst case bound}
The next proposition shows that $\Psi_{t}(\pi_{t}^{\rm IDS})$ is never larger than $|\A|/2$. That is, there is always an action sampling distribution $\pi \in \D(\A)$ such that  $\Delta_{t}(\pi)^2 \leq  (|\A|/2) g_{t}(\pi)$.
As we will show in the coming sections, the ratio between regret and information gain can be much smaller under specific information structures.
\begin{prop}\label{prop: worst case bound}
For any $t \in \mathbb{N}$, $\Psi_{t}(\pi_{t}^{\rm IDS}) \leq  |\A|/2$ almost surely.
\end{prop}
Combining Proposition \ref{prop: worst case bound} with Corollary \ref{cor: general regret bound} shows that $\E \left[\mathrm{Regret}\left(T,  \pi^{\rm IDS} \right) \right]  \leq \sqrt{\frac{1}{2} |\A| H(\alpha_{1})T}$.

\subsubsection{Full information}
Our focus in this paper is on problems with {\it partial feedback}. For such problems, what the decision maker observes depends on the actions selected, which leads to a tension between exploration and exploitation. Problems with full information arise as an extreme point of our formulation where the outcome $Y_{t,a}$ is perfectly revealed by observing $Y_{t,\tilde{a}}$ for some $\tilde{a} \neq a$; what is learned does not depend on the selected action.  The next proposition shows that under full information, the minimal information ratio is bounded by $1/2$.


\begin{prop}\label{prop: full information}
Suppose for each $t\in \mathbb{N}$ there is a random variable $Z_{t}: \Omega \rightarrow \mathcal{Z}$ such that for each $a\in \A$,  $Y_{t,a}= \left( a, Z_{t}  \right)$. Then for all $t\in \mathbb{N}$,  $\Psi_{t}(\pi_{t}^{\rm IDS}) \leq \frac{1}{2}$ almost surely.
\end{prop}

Combining this result with Corollary \ref{cor: general regret bound} shows $\E \left[{\rm Regret}(T, \pi^{{\rm IDS}})  \right] \leq \sqrt{\frac{1}{2} H(\alpha_{1}) T}$. Further, a worst--case bound on the entropy of $\alpha_{1}$ shows that $\E \left[{\rm Regret}(T, \pi^{{\rm IDS}})  \right] \leq \sqrt{\frac{1}{2} \log (|\A|) T}$. \citet{dani2007price} show this bound is order optimal, in the sense that for any time horizon $T$ and number of actions $|\mathcal{A}|$ there exists a prior distribution over $\theta$ under which $\inf_{\pi} \E \left[{\rm Regret}(T, \pi )  \right] \geq c_{0} \sqrt{\log (|\A|) T}$ where $c_{0}$ is a numerical constant that does not depend on $|\A|$ or $T$. The bound here improves upon this worst case bound since $H(\alpha_{1})$ can be much smaller than $\log(|\A|)$ when the prior distribution is informative.

\subsubsection{Linear optimization under bandit feedback}\label{subsec: linear bandit}
The stochastic linear bandit problem has been widely studied (e.g. \cite{dani2008stochastic, rusmevichientong2010linearly, abbasi2011improved}) and is one of the most important examples of a multi-armed bandit problem with ``correlated arms.''  In this setting, each action is associated with a finite dimensional feature vector, and the mean reward generated by an action is the inner product between its known feature vector and some unknown parameter vector. Because of this structure, observations from taking one action allow the decision--maker to make inferences about other actions. The next proposition bounds the minimal information ratio for such problems. 
\begin{prop}\label{prop: linear}
If $\A \subset \mathbb{R}^d$, $\Theta \subset \mathbb{R}^d$,  and
$\E\left[ R_{t,a}| \theta  \right] = a^T \theta$ for each action $a\in \A$,
 then  $\Psi_{t}(\pi_{t}^{\rm IDS}) \leq  d/2$ almost surely for all $t\in \mathbb{N}$.
\end{prop}

This result shows that $\E \left[{\rm Regret}(T, \pi^{\rm IDS})  \right] \leq \sqrt{\frac{1}{2} H(\alpha_{1})d T} \leq \sqrt{\frac{1}{2} \log(|\A|)d T}$ for linear bandit problems. \citet{dani2007price} again show this bound is order optimal in the sense that, for any time horizon $T$ and dimension $d$, when the action set is  $\A = \{0,1 \}^d$ there exists a prior distribution over $\theta$ such that
$\inf_{\pi} \E \left[{\rm Regret}(T, \pi )  \right] \geq c_{0} \sqrt{\log (|\A|) d T}$ where $c_{0}$ is a constant that is independent of $d$ and $T$. The bound here improves upon this worst case bound since $H(\alpha_{1})$ can be much smaller than $\log(|\A|)$ when the prior distribution in informative.

\subsubsection{Combinatorial action sets and ``semi-bandit'' feedback}\label{subsec: semi-bandit}

To motivate the information structure studied here, consider a simple resource allocation problem. There are $d$ possible projects, but the decision--maker can allocate resources to at most $m \leq d$ of them at a time. At time $t$, project $i \in \{1,..,d\}$ yields a random reward $X_{t,i}$, and the reward from selecting a subset of projects $a \in \A \subset \left\{ a' \subset \{0,1,\ldots,d \} :  |a' | \leq m \right\}$ is $m^{-1} \sum_{i \in \A} X_{t,i}$. In the linear bandit formulation of this problem, upon choosing a subset of projects $a$ the agent would only observe the overall reward $m^{-1} \sum_{i \in a} X_{t,i}$. It may be natural instead to assume that the outcome of each selected project $(X_{t,i} : i \in a)$ is observed. This type of observation structure is sometimes called ``semi-bandit'' feedback \cite{audibert2013regret}.

A naive application of Proposition \ref{prop: linear} to address this problem would show $\Psi_{t}^* \leq d/2$. The next proposition shows that since the entire parameter vector $(\theta_{t,i} : i \in a)$ is observed upon selecting action $a$,  we can provide an improved bound on the information ratio.

\begin{prop}\label{prop: semi bandit}
Suppose $\A \subset \left\{ a \subset \{0,1,\ldots,d \} :  |a | \leq m \right\}$, and that there are random variables $(X_{t,i}: t \in \mathbb{N}, i \in \{1,\ldots,d  \})$ such that
\begin{eqnarray*}
Y_{t,a} = \left( X_{t,i} : i \in a \right) &\text{and}& R_{t,a} =\frac{1}{m} \sum_{i\in a }X_{t,i}.
\end{eqnarray*}
Assume that the random variables $\{X_{t,i} : i \in \{1,\ldots,d \}\}$ are independent conditioned on $\hist$ and   $X_{t,i} \in [\frac{-1}{2},\frac{1}{2}]$ almost surely for each $(t,i)$. Then
for all $t\in \mathbb{N}$,  $\Psi_{t}(\pi_{t}^{\rm IDS}) \leq  \frac{d}{2m^2}$ almost surely.
\end{prop}
In this problem, there are as many as $\binom{d}{m}$ actions, but because IDS exploits the structure relating actions to one another, its regret is only polynomial in $m$ and $d$. In particular, combining Proposition \ref{prop: semi bandit} with Corollary \ref{cor: general regret bound} shows $\E \left[{\rm Regret}(T, \pi^{{\rm IDS}})  \right] \leq \frac{1}{m}\sqrt{\frac{d}{2} H(\alpha_{1})T }$. Since $H(\alpha_{1}) \leq \log |\mathcal{A}| = O(m \log (\frac{d}{m}))$ this also yields a bound of order $\sqrt{\frac{d}{m} \log\left(\frac{d}{m}\right)T}$. As shown by \citet{audibert2013regret}, the lower bound\footnote{In their formulation, the reward from selecting action $a$ is $\sum_{i\in a }X_{t,i},$ which is $m$ times larger than in our formulation. The lower bound stated in their paper is therefore of order $\sqrt{mdT}$. They don't provide a complete proof of their result, but note that it follows from standard lower bounds in the bandit literature. In the proof of Theorem 5 in that paper, they construct an example in which the decision maker plays $m$ bandit games in parallel, each with $d/m$ actions. Using that example, and the standard bandit lower bound (see Theorem  3.5 of \citet{bubeck2012regret}), the agent's regret from each component must be at least $\sqrt{\frac{d}{m} T}$, and hence her overall expected regret is lower bounded by a term of order $m\sqrt{\frac{d}{m} T}=\sqrt{mdT}$.} for this problem is of order $\sqrt{\frac{d}{m}T}$, so our bound is order optimal up to a $\sqrt{\log(\frac{d}{m})}$ factor.

\section{Computational methods}
\label{se:computation}

IDS offers an abstract design principle that captures some key qualitative properties of the Bayes-optimal solution while accommodating tractable computation
for many relevant problem classes.  However, additional work is required to design efficient computational methods that implement IDS for specific problem classes.  
In this section, we provide guidance and examples.

We will focus in this section on the problem of generating an action $A_t$ given the posterior distribution over $\theta$ at time $t$.  This sidesteps the problem of 
computing and representing a posterior distribution, which can present its own challenges.  
Though IDS could be combined with approximate Bayesian inference methods, we will focus here on the simpler context in which posterior distributions can be efficiently computed and stored, as is the case when working with tractable finite uncertainty sets or appropriately chosen conjugate priors. It is worth noting, however, that two of our algorithms approximate IDS using samples from the posterior distribution, and this may be feasible through the use of Markov chain Monte Carlo even in cases where the posterior distribution cannot be computed or even stored.

\subsection{Evaluating the information ratio}

Given a finite action set $\A = \{1,\ldots,K\}$, we can view an action distribution $\pi$ as a $K$-dimensional vector of probabilities.  The information ratio can then be written as
$$\Psi_t(\pi) = \frac{\left(\pi^\top \vec{\Delta}\right)^{2}}{\pi^\top \vec{g}},$$
where $\vec{\Delta}$ and $\vec{g}$ are $K$-dimensional vectors with components $\vec{\Delta}_k = \Delta_t(k)$ and $\vec{g}_k = g_t(k)$ for $k \in \A$.  
In this subsection, we discuss the computation of $\vec{\Delta}$ and $\vec{g}$ for use in evaluation of the information ratio.

There is no general efficient procedure for computing $\vec{\Delta}$ and $\vec{g}$ given a posterior distribution, because that would
require computing integrals over possibly high-dimensional spaces.  Such computation can often be carried out efficiently 
by leveraging the functional form of the specific posterior distribution and often require numerical integration.  
In order to illustrate the design of problem-specific computational procedures, we will present two simple examples in this subsection.

We begin with a conceptually simple model involving finite uncertainty sets.
\begin{example}
{\bf (finite sets)}
Consider a problem in which $\theta$ takes values in $\Theta = \{1,\ldots,L\}$, the action set is $\mathcal{A} = \{1,\ldots,K\}$, 
the observation set is $\mathcal{Y} = \{1,\ldots,N\}$, and the reward function $R:\mathcal{Y}\mapsto \mathbb{R}$ is arbitrary.
Let $p_1$ be the prior probability mass function of $\theta$ and let $q_{\theta,a}(y)$ be the probability, conditioned on $\theta$, of observing $y$ when action $a$ is selected.
\end{example}
\noindent Note that the posterior probability mass function $p_t$, conditioned on observations made prior to period $t$, can be computed recursively via Bayes' rule:
$$p_{t+1}(\theta) \leftarrow \frac{p_t(\theta) q_{\theta,A_t}(Y_{t,A_t})}{\sum_{\theta' \in \Theta} p_t(\theta')  q_{\theta',A_t}(Y_{t,A_t})}.$$
Given the posterior distribution $p_t$ along with the model parameters $(L, K, N, R, q)$, Algorithm \ref{alg: finiteIR} computes $\vec{\Delta}$ and $\vec{g}$.
Line 1 computes the optimal action for each value of $\theta$.  Line 2 calculates the probability that
each action is optimal.  Line 3 computes the marginal distribution of $Y_{1,a}$ and line 4 computes the joint probability mass function of $(A^*, Y_{1,a})$. Lines 5 and 6 use the aforementioned probabilities to compute $\vec{\Delta}$ and $\vec{g}$.

\begin{algorithm}[H]
\caption{$\text{finiteIR}(L, K, N, R, p, q)$}\label{alg: finiteIR}
\begin{algorithmic}[1]
\STATE $\Theta_a \leftarrow \{ \theta | a= \arg\max_{a'} \sum_y q_{\theta,a'}(y) R(y) \} \qquad \forall \theta$
\STATE $p(a^*) \leftarrow \sum_{\theta \in \Theta_{a^*}} p(\theta) \qquad \forall a^*$
\STATE $p_{a}(y) \leftarrow \sum_{\theta} p(\theta) q_{\theta,a}(y) \qquad \forall a,y,\theta$
\STATE $p_{a}(a^*, y) \leftarrow \frac{1}{p(a^*)}\sum_{\theta \in \Theta_{a^*}} q_{\theta,a}(y) \qquad \forall a,y,a^*$
\STATE $R^* \leftarrow  \sum_{a} \sum_{\theta \in \Theta_a}\sum_{y} p(\theta) q_{\theta,a}(y) R(y)$
\STATE $\vec{g}_a \leftarrow \sum_{a^*, y} p_{a}(a^*, y) \log \frac{p_{a}(a^*, y)}{p(a^*)p_{a}(y)}
\qquad \forall a$
\STATE $\vec{\Delta}_a \leftarrow R^* - \sum_\theta p(\theta) \sum_y q_{\theta,a}(y) R(y) \qquad \forall a$
\RETURN $\vec{\Delta}, \vec{g}$
\end{algorithmic}
\end{algorithm}

Next, we consider the beta-Bernoulli bandit.
\begin{example}\label{ex: beta-Bernoulli}
{\bf (beta-Bernoulli bandit)}
Consider a multi-armed bandit problem with  binary rewards: $\A=\{1,\ldots,K\}$, $\mathcal{Y} = \{0,1\}$, 
and $R(y) = y$.  Model parameters $\theta \in \mathbb{R}^K$ specify the mean reward $\theta_a$ of each action $a$.  Components
of $\theta$ are independent and each beta-distributed with prior parameters $\beta_1^1, \beta_1^2 \in \mathbb{R}^K_+$
\end{example}
\noindent Because the beta distribution is a conjugate prior for the Bernoulli distribution, the posterior distribution of each $\theta_a$ is a beta distribution. The posterior parameters 
$\beta^1_{t,a}, \beta^2_{t,a} \in \mathbb{R}_+$ can be computed recursively:
$$(\beta^1_{t+1,a}, \beta^2_{t+1,a}) \leftarrow \left\{\begin{array}{ll}
(\beta^1_{t,a} + Y_{t,a}, \beta^2_{t,a} + (1-Y_{t,a})) \qquad &\text{if } A_t = a \\
(\beta^1_{t,a}, \beta^2_{t,a}) \qquad &\text{otherwise.}
\end{array}\right.$$
Given the posterior parameters $(\beta_t^1,\beta_t^2)$, Algorithm \ref{alg: betaBernoulliIR} computes $\vec{\Delta}$ and $\vec{g}$.

Line 5 of the algorithm computes the posterior probability mass function of  $A^*$.  It is easy to derive the expression used:
\begin{eqnarray*}
\mathbb{P}_{t}(A^* = a)
&=& \mathbb{P}_{t}\left( \bigcap_{a' \neq a}  \{\theta_{a'} \leq \theta_{a} \}   \right) \\
&=& \int_{0}^{1} f_a(x) \mathbb{P}_t\left( \bigcap_{a' \neq a}  \{  \theta_{a'} \leq x \} \bigg\vert \theta_a = x   \right) dx \\
&=&  \int_{0}^{1} f_a(x) \left( \prod_{a' \neq a} F_{a'}(x) \right)dx \\
&=& \int_{0}^{1} \left[\frac{f_a(x)}{F_a(x)}\right] \overline{F}(x)dx,
\end{eqnarray*}
where $f_a$, $F_a$, and $\overline{F}$ are defined as in lines 1-3 of the algorithm, with arguments $(K, \beta_t^1,\beta_t^2)$.
Using expressions that can be derived in a similar manner, for each pair of actions Lines 6-7 compute
$M_{a'|a}:=\E_t\left[\theta_{a'} | \theta_a = \max_{a''} \theta_{a''}\right]$, the expected value of $\theta_{a'}$ given that action $a$ is optimal.
Lines 8-9 computes the expected reward of the optimal action 
$\rho^* = \E_t\left[\max_a \theta_a\right]$ and uses that to compute, for each action, 
$$\vec{\Delta}_a = \E_t\left[ \max_{a'} \theta_a - \theta_a \right] = \rho^* - \frac{\beta^1_{t,a}}{ (\beta^1_{t,a}+\beta^2_{t,a})}.$$
Finally, line 10 computes $\vec{g}$.  The expression makes use of the following fact, which is a consequence of standard properties of mutual 
information\footnote{Some details related to the derivation of this fact when $Y_{t,a}$ is a general random variable can be found in the appendix of \citet{russo2016info}.}:
\begin{equation}\label{eq: info gain as expected divergence}
I_{t}(A^*; Y_{t,a})= \sum_{a^*\in \A}\Prob_{t}(A^*=a^*) \DKL \left(\Prob_{t}(Y_{t,a}=\cdot \vert A^*=a^*) \,||\,  \Prob_{t}(Y_{t,a}=\cdot )  \right).
\end{equation}
That is, the mutual information between $A^*$ and $Y_{t,a}$ is the expected Kullback-Leibler divergence between the posterior predictive distribution $\Prob_{t}(Y_{t,a}=\cdot)$ and the predictive distribution conditioned on the identity of the optimal action $\Prob_{t}(Y_{t,a}=\cdot \vert A^*=a^*)$.  For our beta-Bernoulli model,
the information gain $\vec{g}_a$ is the expected Kullback-Leibler divergence between a Bernoulli distribution with mean 
$M_{a|A^*}$ and the posterior distribution at action $a$, which is Bernoulli with parameter $\beta^1_{t,a}/ (\beta^1_{t,a}+\beta^2_{t,a})$.

Algorithm \ref{alg: betaBernoulliIR}, as we have presented it, is somewhat abstract and can not readily be implemented on a computer.
In particular, lines 1-4 require computing and storing functions of a continuous variable and several lines require integration of continuous functions.
However, near-exact approximations can be efficiently generated by evaluating integrands at discrete 
grid of points $\{x^{1},\ldots,x^{n}\} \subset [0,1]$.  The values of $f_a(x), F_a(x), G_a(x)$ and $\overline{F}(x)$ can be computed and stored for each 
value in this grid.  The compute time can also be reduced via memoization, since values change only for 
one action per time period.  The compute time of such an implementation scales with 
$K^2 n$ where $K$ is the number of actions and $n$ is the number of points used in the discretization of $[0,1]$.
The bottleneck is Line 7.

\begin{algorithm}[H]
\caption{$\text{betaBernoulliIR}(K, \beta^1,\beta^2)$}\label{alg: betaBernoulliIR}
\label{algo: beta}
\begin{algorithmic}[1]
\STATE{$f_a(x) \leftarrow {\rm beta.pdf}(x \vert  \beta^1_a, \beta^2_a) \qquad \forall a,x$}
\STATE{$F_a(x) \leftarrow {\rm beta.cdf}(x \vert  \beta^1_a, \beta^2_a) \qquad \forall a,x$}
\STATE{$\overline{F}(x) \leftarrow \prod_a F_a(x) \qquad \forall x$}
\STATE{$G_a(x) \leftarrow \intop_{0}^{x} y f_a(y)dy \qquad \forall a,x$}
\STATE{$p^*(a) \leftarrow \intop_{0}^{1} \left[\frac{f_a(x)}{F_a(x)}\right] \overline{F}(x)dx \qquad \forall a$}
\STATE{$M_{a|a} \leftarrow \frac{1}{p^*(a)}\intop_{0}^{1} \left[\frac{x f_a(x)}{F_a(x)}\right] \overline{F}(x)dx \qquad \forall a$}
\STATE{$M_{a'|a} \leftarrow  \frac{1}{p^*(a)} \intop_{0}^{1} \left[\frac{f_a(x) \overline{F}(x)}{F_a(x) F_{a'}(x)}\right]  G_{a'}(x) dx \qquad \forall a,a'\neq a$}
\STATE{$\rho^* \leftarrow \sum_{a} p^*(a) M_{a|a}$}
\STATE{$\vec{\Delta}_a \leftarrow \rho^*- \frac{\beta_a^1}{\beta_a^1+\beta_a^2} \qquad \forall a$}
\STATE{$\vec{g}_a \leftarrow \sum_{a'} p^*(a') \left(M_{a|a'} \log \left(M_{a|a'} (\beta^1_a+\beta^2_a) / \beta^1_a \right)+(1-M_{a|a'}) \log \left((1-M_{a|a'})(\beta^1_a+\beta^2_a) / \beta^2_a\right)\right) \quad \forall a$}
\RETURN $\vec{\Delta}, \vec{g}$
\end{algorithmic}
\end{algorithm}

\subsection{Optimizing the information ratio}

Let us now discuss how to generate an action given $\vec{\Delta}$ and $\vec{g} \neq 0$.
If $\vec{g}=0$, the optimal action is known with certainty, and therefore the action selection problem is trivial.  Otherwise, 
IDS selects an action by solving
\begin{equation}
\label{eq: objective function}
\min_{\pi \in \mathcal{S}_K} \frac{\left(\pi^\top \vec{\Delta}\right)^{2}}{\pi^\top \vec{g}}
\end{equation}
where $\mathcal{S}_K = \{\pi \in \mathbb{R}^K_+: \sum_k \pi_k = 1\}$ is the $K$-dimensional unit simplex, and samples from the resulting distribution $\pi$.

The following result establishes that \eqref{eq: objective function} is a convex optimization problem and, surprisingly, has an optimal solution with at most two non-zero components. Therefore, while IDS is a randomized policy, it suffices to randomize over two actions.
\begin{prop}\label{prop: support at most 2}
For all $\vec{\Delta}, \vec{g} \in \mathbb{R}^K_+$ such that $\vec{g} \neq 0$,
the function $\pi \mapsto \left(\pi^\top \vec{\Delta}\right)^{2}/ \pi^\top \vec{g}$ is convex on $\left\{\pi \in \mathbb{R}^K : \pi^\top \vec{g} > 0  \right\}$.
Moreover, this function is minimized over $\mathcal{S}_K$ by some $\pi^*$ for which $|\{k : \pi^*_k >0 \}| \leq 2$.
\end{prop}

Algorithm \ref{alg: chooseAction} leverages Proposition \ref{prop: support at most 2} to efficiently choose an action in a manner that minimizes \eqref{eq: definition of ratio policy}.  The algorithm takes as input $\vec{\Delta} \in \mathbb{R}^{K}_+$ and $\vec{g} \in \mathbb{R}^{K}_+$, which provide the expected regret and information gain of each action. The sampling distribution that minimizes \eqref{eq: definition of ratio policy} is computed by iterating over all pairs of actions $(a,a')\in \A \times \A$, and for each, computing the probability $q$ that minimizes the information ratio among distributions that sample $a$ with probability $q$ and $a'$ with probability $1-q$.  This one-dimensional optimization problem requires little computation since the objective is convex; $q$ can be computed by solving for the first-order necessary condition or approximated by a bisection method.  The compute time of this algorithm scales with $K^2$.

\begin{algorithm}[H]
\caption{IDSAction($K, \vec{\Delta}, \vec{g}$)}\label{alg: chooseAction}
\begin{algorithmic}[1]
\STATE{$q_{a,a'} \leftarrow \arg\min_{q' \in [0,1]} \left[q' \vec{\Delta}_a +(1-q') \vec{\Delta}_{a'}\right]^2 / \left[ q' \vec{g}_a +(1-q') \vec{g}_{a'} \right]  \qquad \forall a < K, a' > a$}
\STATE{$(a^*,a^{**}) \leftarrow \arg\min_{a < K,a' > a} \left[q_{a,a'} \vec{\Delta}_a +(1-q_{a,a'}) \vec{\Delta}_{a'}\right]^2 / \left[q_{a,a'} \vec{g}_a +(1-q_{a,a'}) \vec{g}_{a'} \right]$}
\STATE{Sample $b \sim {\rm Bernoulli}(q_{a^*, a^{**}})$}
\RETURN{$b a^* + (1-b) a^{**}$}
\end{algorithmic}
\end{algorithm}

\subsection{Approximating the information ratio}\label{subsec: approximating the information ratio}

Though reasonably efficient algorithms can be devised to implement IDS for various problem classes, some applications, such as those arising in high-throughput web services, 
call for extremely fast computation.  As such, it is worth considering approximations to the information ratio that retain salient features while
enabling faster computation.  In this section, we discuss some useful approximation concepts.

The dominant source of complexity in computing  $\vec{\Delta}$ and $\vec{g}$ is in the calculation of requisite integrals,
which can require integration over high-dimensional spaces.
One approach to addressing this challenge is to replace integrals with sample-based estimates.  
Algorithm \ref{alg: SampleIR} does this.  In addition to the number of actions $K$ and routines for evaluation $q$ and $R$, the algorithm takes as 
input $M$ representative samples of $\theta$.  In the simplest use scenario, these would be independent samples drawn from the posterior distribution.
The steps correspond to those of Algorithm \ref{alg: finiteIR}, but with the set of possible models approximated by the set of representative samples. For many problems, even when exact computation
of $\vec{\Delta}$ and $\vec{g}$ is intractable due to required integration over high-dimensional spaces, Algorithm \ref{alg: finiteIR} can generate
close approximations from a moderate number of samples $M$.

\begin{algorithm}[H]
\caption{$\text{SampleIR}(K, q, R, M, \theta^1,\ldots,\theta^M)$}\label{alg: SampleIR}
\begin{algorithmic}[1]
\STATE $\hat{\Theta}_a \leftarrow \{m | a= \arg\max_{a'} \sum_{y} q_{\theta^m, a'}(y)R(y) \}$
\STATE $\hat{p}(a^*) \leftarrow |\hat{\Theta}_{a^*}|/M \qquad \forall a^*$
\STATE $\hat{p}_{a}(y) \leftarrow \sum_{m} q_{a,\theta^m}(y) /M \qquad \forall y$
\STATE $\hat{p}_{a}(a^*, y) \leftarrow \sum_{m\in \Theta_a} q_{a,\theta^m}(y) /M \qquad \forall a^*, y$
\STATE $\hat{R}^* \leftarrow\sum_{a,y} \hat{p}_{a}(a, y)R(y)$
\STATE $\vec{g}_a \leftarrow \sum_{a^*, y} \hat{p}_{a}(a^*, y) \log \frac{\hat{p}_{a}(a^*, y)}{\hat{p}(a^*)\hat{p}_{a}(y)}
\qquad \forall a$
\STATE $\vec{\Delta}_a \leftarrow R^* - M^{-1}\sum_{m} \sum_y q_{\theta^m,a}(y) R(y) \qquad \forall a$
\RETURN $\vec{\Delta}, \vec{g}$
\end{algorithmic}
\end{algorithm}

The information ratio is designed to effectively address indirect information, cumulating information, and irrelevant information, for a very broad class
of learning problems.  It can sometimes be helpful to replace the information ratio with alternative information measures that adequately address these
issues for more specialized classes of problems.  As an example, we will introduce the variance-based information ratio,
which is suitable for some problems with bandit feedback, satisfies our regret bounds for such problems, and can facilitate design of more efficient numerical methods.

To motivate the variance-based information ratio, note that when rewards our bounded, with $R(y)\in [0,1]$ for all $y$, 
our information measure term is lower-bounded according to
\begin{eqnarray*}
g_{t}(a) 
&=&I_{t}(A^*; Y_{t,a}) \\
&=& \sum_{a^*\in \A}\Prob_{t}(A^*=a^*) \DKL \left(\Prob_{t}(Y_{t,a}=\cdot \vert A^*=a^*) \,||\,  \Prob_{t}(Y_{t,a}=\cdot )  \right)   \\
&\geq & \sum_{a^*\in \A}\Prob_{t}(A^*=a^*) \DKL \left(\Prob_{t}(R_{t,a}=\cdot \vert A^*=a^*) \,||\,  \Prob_{t}(R_{t,a}=\cdot )  \right) \nonumber \\
&\overset{(a)}{\geq} & 2\sum_{a^*\in \A}\Prob_{t}(A^*=a^*)( \E_{t}[R_{t,a}|A^*=a^*] -\E_{t}[R_{t,a}])^2 \nonumber\\
&=& 2\E_{t}[(\E_{t}[R_{t,a}|A^*] -\E_{t}[R_{t,a}])^2] \nonumber \\
&=& 2{\rm Var}_{t}( \E_{t}[R_{t,a}|A^*]),
\end{eqnarray*}
where ${\rm Var}_{t}(X) = \E_{t}[(X- \E_{t}[X])^2]$ denotes the variance of $X$ under the posterior distribution.  Inequality (a) is a simple corollary of Pinsker's inequality, and is given as Fact 9 in \citet{russo2016info}. Let $v_t(a) :=  {\rm Var}_{t}( \E_{t}[R_{t,a}|A^*])$, 
which represents the variance of the conditional expectation $\E_{t}[R_{t,a}|A^*]$ under the posterior distribution. 
This measures how much the expected reward generated by action $a$ varies depending on the identity of the optimal action $A^*$. 
The above lower bound on mutual information indicates that actions with high variance $v_t(a)$ must yield substantial information about which action is optimal.
It is natural to consider an approximation to IDS that uses a variance-based information ratio:
$$\min_{\pi \in \mathcal{S}_K} \frac{\left(\pi^\top \vec{\Delta}\right)^2}{\pi^\top \vec{v}},$$
where $\vec{v}_a = v_t(a)$.

While variance-based IDS will not minimize the information ratio, the next proposition establishes that it satisfies the bounds 
on the information ratio given by Propositions \ref{prop: worst case bound} and \ref{prop: linear}.  
\begin{prop}\label{prop: information ratio bounds for v-ids}
Suppose $\sup_{y} R(y) - \inf_{y} R(y) \leq 1$  and
\[
\pi_{t} \in \arg \min _{\pi \in \mathcal{S}_K} \frac{\Delta_{t}(\pi)^2}{v_{t}(\pi)}.
\] 
Then $\Psi_{t}(\pi_t) \leq |\A|/2$. Moreover, if  $\A \subset \mathbb{R}^d$, $\Theta \subset \mathbb{R}^d$,  and
$\E\left[ R_{t,a}| \theta  \right] = a^T \theta$ for each action $a\in \A$, then $\Psi_{t}(\pi_t) \leq d/2$.
\end{prop}

We now consider a couple examples that illustrate computation of $\vec{v}$ and benefits of using this approximation.  Our first example
is the independent Gaussian bandit problem.
\begin{example}
{\bf (independent Gaussian bandit)}
Consider a multi-armed bandit problem with $\A=\{1,\ldots,K\}$, $\mathcal{Y} = \mathbb{R}$, and $R(y) = y$.  
Model parameters $\theta \in \mathbb{R}^K$ specify the mean reward $\theta_a$ of each action $a$.
Components of $\theta$ are independent and Gaussian-distributed, with prior means $\mu_1 \in \mathbb{R}^K$ and covariances $\sigma^2_1 \in \mathbb{R}^{K}$.
When an action $A_t$ is applied, the observation $Y_t$ is drawn independently from $N(\theta_{A_t},\eta^2)$.
\end{example}
\noindent The posterior distribution of $\theta$ is Gaussian, with independent components.  Parameters can be computed recursively according to
\begin{eqnarray*}
\mu_{t+1, a} &\leftarrow& \left\{\begin{array}{ll}
\left(\frac{\mu_{t,a}}{\sigma_{t,a}^2} + \frac{Y_{t,a}}{\eta^2}\right)/\left(\frac{1}{\sigma_{t,a}^2} + \frac{1}{\eta^2}\right) \qquad & \text{if } A_t = a \\
\mu_{t, a} \qquad & \text{otherwise}.
\end{array}\right.   \\    
\sigma_{t+1, a}  &\leftarrow& \left\{\begin{array}{ll}
\left(\frac{1}{\sigma_{t,a}^2} + \frac{1}{\eta^2}\right)^{-1} \qquad & \text{if } A_t = a \\
\sigma_{t, a} \qquad & \text{otherwise}.
\end{array}\right.    
\end{eqnarray*}
Given arguments $(K, \mu_t, \sigma_t)$, Algorithm \ref{alg: gaussianVIR} computes $\vec{\Delta}$ and $\vec{v}$ for the independent Gaussian bandit problem.
Note that this algorithm is very similar to Algorithm \ref{alg: betaBernoulliIR}, which was designed for the beta-Bernoulli bandit. 
One difference is that Algorithm \ref{alg: gaussianVIR} computes the variance-based information measure.
In addition, the Gaussian distribution exhibits special structure that simplifies the computation of 
$M_{a'|a} := \mathbb{E}_t\left[\theta_{a'} | \theta_a = \max_{a''} \theta_{a''} \right].$
In particular, the computation of $M_{a'|a}$ uses the following closed form expression for the expected value of a truncated Gaussian distribution with mean $\tilde{\mu}$
and variance $\tilde{\sigma}^2$:
$$\E\left[X \vert X \leq x \right] = \tilde{\mu} - \tilde{\sigma} \phi\left( \frac{x-\tilde{\mu}}{\tilde{\sigma}} \right) / \Phi\left( \frac{x-\tilde{\mu}}{\tilde{\sigma}} \right)=\tilde{\mu} - \tilde{\sigma}^2 f(x) /F(x),$$
where $X \sim N(\tilde{\mu}, \tilde{\sigma}^2)$ and $f$ and $F$ are the probability density and cumulative distribution functions.

The analogous calculation that would be required to compute the standard information ratio is more complex.

\begin{algorithm}[H]
\caption{$\text{independentGaussianVIR}(K, \mu, \sigma)$}\label{alg: gaussianVIR}
\begin{algorithmic}[1]
\STATE{$f_a(x) \leftarrow {\rm Gaussian.pdf}(x \vert  \mu_a, \sigma^2_a) \qquad \forall a,x$}
\STATE{$F_a(x) \leftarrow {\rm Gaussian.cdf}(x \vert  \mu_a, \sigma^2_a) \qquad \forall a,x$}
\STATE{$\overline{F}(x) \leftarrow \prod_a F_a(x) \qquad \forall x$}
\STATE{$p^*(a) \leftarrow \intop_{0}^{1} \left[\frac{f_a(x)}{F_a(x)}\right] \overline{F}(x)dx \qquad \forall a$}
\STATE{$M_{a|a} \leftarrow \frac{1}{p^*(a)}\int_{-\infty}^{\infty} \left[\frac{x f_a(x)}{F_a(x)}\right] \overline{F}(x)dx \qquad \forall a$}
\STATE{$M_{a'|a} \leftarrow \mu_{a'} - \frac{\sigma_{a'}^2}{p^*(a)} \int_{-\infty}^{\infty} \left[\frac{f_a(x) f_{a'}(x)}{F_a(x) F_{a'}(x)}\right]  \overline{F}(x) dx \qquad \forall a,a'\neq a$}
\STATE{$\rho^* \leftarrow \sum_{a} p^*(a) M_{a|a}$}
\STATE{$\Delta_a \leftarrow \rho^*- \mu_a \qquad \forall a$}
\STATE{$v_a \leftarrow \sum_{a'} p^*(a') \left(M_{a|a'} - \mu_a \right)^2 \qquad \forall a$}
\RETURN $\vec{\Delta}, \vec{v}$
\end{algorithmic}
\end{algorithm}

We next consider the linear bandit problem.
\begin{example}
{\bf (linear bandit)}
Consider a multi-armed bandit problem with $\A=\{1,\ldots,K\}$, $\mathcal{Y} = \mathbb{R}$, and $R(y) = y$.  
Model parameters $\theta \in \mathbb{R}^K$ are drawn from a Gaussian prior with mean $\mu_1$ and covariance matrix $\Sigma_1$.
There is a known matrix $\Phi=[\Phi_{1},\cdots, \Phi_{K}] \in \mathbb{R}^{d\times K}$ such that, when an action $A_t$ is applied, the observation $Y_{t,A_t}$
is drawn independently from $N(\Phi_{A_t} \theta, \eta^2)$, where $\Phi_{A_t}$ denotes the $A_t$th column of $\Phi$.
\end{example}
\noindent The posterior distribution of $\theta$ is Gaussian and can be computed recursively:
$$\mu_{t+1} = (\Sigma_t^{-1} + \Phi_{A_t}\Phi_{A_t}^{\top}   /\eta^2)^{-1} (\Sigma_t^{-1} \mu_t + Y_{t,A_t} \Phi_{A_t} /\eta^2)$$
$$\Sigma_{t+1} = (\Sigma_t^{-1} + \Phi_{A_t}\Phi_{A_t}^{\top}/\eta^2)^{-1}.$$

We will develop an algorithm that leverages the fact that, for the linear bandit, $v_t(a)$ takes on a particularly simple form:
\begin{eqnarray*}
v_t(a)
&=& {\rm Var}_{t}( \E_{t}[R_{t,a}|A^*]) \\
&=& {\rm Var}_{t}( \E_{t}[\Phi_a^{\top} \theta |A^*]) \\
&=& {\rm Var}_{t}(\Phi_a^{\top} \E_{t}[\theta |A^*]) \\
&=& \Phi_a^\top \E_t[(\mu_t^{A^*} - \mu_t) (\mu_t^{A^*} - \mu_t)^\top] \Phi_a \\
&=& \Phi_a^\top L_t \Phi_a,
\end{eqnarray*}  
where $\mu_t^a = \E_{t}[\theta |A^*=a]$ and $L_t = \E_t[(\mu_t^{A^*} - \mu_t) (\mu_t^{A^*} - \mu_t)^\top]$.
Algorithm \ref{alg: linearVIR} presents a sample-based approach to computing $\vec{\Delta}$ and $\vec{v}$.  In addition to model dimensions
$K$ and $d$ and the problem data matrix $\Phi$, the algorithm takes as input $M$ representative values of $\theta$, which in the simplest use scenario, would be
independent samples drawn from the posterior distribution $N(\mu_t,\Sigma_t)$.
The algorithm approximates posterior means $\mu_t$ and $\mu^a_t$ as well as $L_t$ by averaging suitable expressions
over these samples.  Due to the quadratic structure of $v_t(a)$, these calculations are substantially simpler than those
that would be carried out by Algorithm \ref{alg: SampleIR}, specialized to this context.

\begin{algorithm}[H]
\caption{$\text{linearSampleVIR}(K,d,M,\theta^1,\ldots,\theta^M)$}\label{alg: linearVIR}
\label{alg: linearIDS}
\begin{algorithmic}[1]
\STATE $\hat{\mu} \leftarrow \sum_m \theta^m / M$
\STATE $\hat{\Theta}_a \leftarrow \{m: (\Phi^{\top} \theta^m)_a = \max_{a'} (\Phi \theta^m)_{a'}\} \qquad \forall a$
\STATE $\hat{p}^*(a) \leftarrow |\hat{\Theta}_a| / M \qquad \forall a$
\STATE $\hat{\mu}^a \leftarrow \sum_{\theta \in \hat{\Theta}_a} \theta / |\hat{\Theta}_a| \qquad \forall a$
\STATE{$\hat{L} \leftarrow \sum_a \hat{p}^*(a) \left(\hat{\mu}^a - \hat{\mu} \right)\left(\hat{\mu}^a - \hat{\mu} \right)^\top$}
\STATE{$\rho^* \leftarrow \sum_a \hat{p}^*(a) \Phi_a^{\top} \hat{\mu}^a$}
\STATE{$\vec{v}_a \leftarrow \Phi_a^\top \hat{L} \Phi_a^{\top} \qquad \forall a$}
\STATE{$\vec{\Delta}_a \leftarrow \rho^* - \Phi_a^{\top} \hat{\mu} \qquad \forall a$}
\RETURN $\vec{\Delta}, \vec{v}$
\end{algorithmic}
\end{algorithm}

It is interesting to note that Algorithms \ref{alg: SampleIR} and \ref{alg: linearVIR} do not rely on any special structure in the posterior distribution.
Indeed, these algorithms should prove effective regardless of the form taken by the posterior.
This points to a broader opportunity to use IDS or approximations to address complex models for which posteriors can not be
efficiently computed or even stored, but for which it is possible to generate posterior samples via Markov chain Monte Carlo methods.
We leave this as a future research opportunity.

\section{Computational results}

This section presents computational results from experiments that evaluate the effectiveness of information-directed sampling in comparison to alternative algorithms.  In Section 
\ref{sebsec: alternatives}, we showed that alternative approaches like UCB algorithms, Thompson sampling, and the knowledge gradient algorithm can perform very poorly when faced with complicated information structures and for this reason can be dramatically outperformed by IDS. In this section, we focus instead on simpler settings where current approaches are extremely effective. We find that even for these simple and widely studied settings, information-directed sampling displays state-of-the-art performance.  For each experiment, the algorithm used to implement IDS is presented in the previous section.

IDS, Thompson sampling (TS), and some UCB algorithms, do not take the horizon $T$ as input, and are instead designed to work well for all sufficiently long horizons. Other algorithms we simulate were optimized for the particular horizon of the simulation trial. The KG and KG* algorithms in particular, treat the simulation horizon as known, and explore less aggressively in later periods. We have tried to clearly delineate which algorithms are optimized for simulation horizon. We believe one can also design variants of IDS, TS, and UCB algorithms that reduce exploration as the time remaining diminishes, but leave this for future work. 

\subsection{Beta-Bernoulli bandit}\label{subsec: Bernoulli experiment}
Our first experiment involves a multi-armed bandit problem with independent arms and binary rewards. The mean reward of each arm is drawn from ${\rm Beta}(1,1)$, which is the uniform distribution, and the means of separate arms are independent. Figure \ref{fig: bernoulli_regret_plot} and Table \ref{table: bernoulli} present the results of 1000 independent trials of an experiment with 10 arms and a time horizon of 1000.  We compared the performance of IDS to that of six other algorithms, and found that it had the lowest average regret of 18.0.

The UCB1 algorithm of \citet{auer2002finite} selects the action $a$ which maximizes the upper confidence bound $ \hat{\theta}_{t}(a)+ \sqrt{2 \log(t) / N_{t}(a)}$ where $\hat{\theta}_{t}(a)$ is the empirical average reward from samples of action $a$ and $N_{t}(a)$ is the number of samples of action $a$ up to time $t$. The average regret of this algorithm is 130.7, which is dramatically larger than that of IDS. For this reason UCB1 is omitted from Figure \ref{fig: bernoulli_regret_plot}.

The confidence bounds of UCB1 are constructed to facilitate theoretical analysis. For practical performance \citet{auer2002finite} proposed using an algorithm called UCB-Tuned. This algorithm selects the action $a$ which maximizes the upper confidence bound $ \hat{\theta}_{t}(a)+ \sqrt{\min\{1/4 \, , \,  \overline{V}_{t}(a) \}\log(t) / N_{t}(a)}$, where $ \overline{V}_{t}(a)$ is an upper bound on the variance of the reward distribution at action $a$. While this method dramatically outperforms UCB1, it is still outperformed by IDS. The MOSS algorithm of \citet{audibert2009minimax} is similar to UCB1 and UCB--Tuned, but uses slightly different confidence bounds. It is known to satisfy regret bounds  for this problem that are minimax optimal up to a numerical constant factor.

In previous numerical experiments \cite{scott2010modern, kaufmann2012thompson, kaufmann2012bayesian, chapelle2011empirical}, Thompson sampling and Bayes UCB exhibited state-of-the-art performance for this problem. Each also satisfies strong theoretical guarantees, and is known to be asymptotically optimal in the sense defined by \citet{lai1985asymptotically}. Unsurprisingly, they are the closest competitors to IDS. The Bayes UCB algorithm, studied in \citet{kaufmann2012bayesian}, constructs upper confidence bounds based on the quantiles of the posterior distribution: at time step $t$ the upper confidence bound  at an action is the $1-\frac{1}{t}$ quantile of the posterior distribution of that action\footnote{Their theoretical guarantees require choosing a somewhat higher quantile, but the authors suggest choosing this quantile, and use it in their own numerical experiments.}.

A somewhat different approach is the knowledge gradient (KG) policy of \citet{powell2012optimal}, which uses a one-step lookahead approximation to the value of information to guide experimentation. For reasons described in Section \ref{subsec: expected improvement}, KG does not explore sufficiently to identify the optimal arm in this problem, and therefore its regret grows linearly with time. Because KG explores very little, its realized regret is highly variable, as depicted in Table \ref{table: bernoulli}. In 200 out of the 2000 trials, the regret of KG was lower than .7, reflecting that the best arm was almost always chosen. In the worst 200 out of the 2000 trials, the regret of KG was larger than 159. 

KG is particularly poorly suited to problems with discrete observations and long time horizons. The KG* heuristic of \citet{ryzhov2010robustness} offers much better performance in some of these problems. At time $t$, KG* calculates the value of sampling an arm 
for $M \in \{1,..,T-t\}$ periods and choosing the arm with the highest posterior mean in subsequent periods. It selects an action by maximizing this quantity over all possible arms and possible exploration lengths $M$. Our simulations require computing $T=1,000$ decisions per trial, and a direct implementation of KG* requires order $T^3$ basic operations per decision. To enable efficient simulation, we use a heuristic approach to computing KG* proposed by \citet{kaminski2015refined}. The approximate KG* algorithm we implement uses golden section search to maximize a non-concave function, but is still empirically effective. 

Finally, as demonstrated in Figure \ref{fig: bernoulli_regret_plot},  variance-based IDS offers performance very similar to standard IDS for this problem.

It is worth pointing out that, although Gittins' indices characterize the Bayes optimal policy for infinite horizon discounted problems, the finite horizon formulation considered here is computationally intractable \cite{gittins2011multi}. A similar index policy \cite{nino2011computing} designed for finite horizon problems could be applied as a heuristic in this setting. However, with long time horizons, the associated computational requirements become onerous.

\begin{table}[]
	\centering
	\begin{tabular}{l|l|l|l|l|l|l|l|l|l|}
		\cline{2-10}
		\textbf{} & \multicolumn{6}{c|}{\textbf{Time Horizon Agnostic}} & \multicolumn{3}{c|}{\textbf{\begin{tabular}[c]{@{}l@{}}Optimized For \\ Time Horizon\end{tabular}}} \\ \hline
		\multicolumn{1}{|l|}{\textbf{Algorithm}} & \textbf{IDS} & \textbf{V-IDS} & \textbf{TS} & \textbf{\begin{tabular}[c]{@{}l@{}}Bayes\\ UCB\end{tabular}} & \textbf{UCB1} & \textbf{\begin{tabular}[c]{@{}l@{}}UCB-\\ Tuned\end{tabular}} & \textbf{MOSS} & \textbf{KG} & \textbf{KG*} \\ \hline
		\multicolumn{1}{|l|}{Mean Regret} & 18.0 & 18.1 & 28.1 & 22.8 & 130.7 & 36.3 & 46.7 & 51.0 & 18.4 \\ \hline
		\multicolumn{1}{|l|}{Standard Error} & 0.4 & 0.4 & 0.3 & 0.3 & 0.4 & 0.3 & 0.2 & 1.5 & 0.6 \\ \hline
		\multicolumn{1}{|l|}{Quantile .10} & 3.6 & 5.2 & 13.6 & 8.5 & 104.2 & 24.0 & 36.2 & 0.7 & 2.9 \\ \hline
		\multicolumn{1}{|l|}{Quantile .25} & 7.4 & 8.1 & 18.0 & 12.5 & 117.6 & 29.2 & 40.0 & 2.9 & 5.4 \\ \hline
		\multicolumn{1}{|l|}{Quantile .50} & 13.3 & 13.5 & 25.3 & 20.1 & 131.6 & 35.2 & 45.2 & 11.9 & 8.7 \\ \hline
		\multicolumn{1}{|l|}{Quantile .75} & 22.5 & 22.3 & 35.0 & 30.6 & 144.8 & 41.9 & 51.0 & 82.3 & 16.3 \\ \hline
		\multicolumn{1}{|l|}{Quantile .90} & 35.6 & 36.5 & 46.4 & 40.5 & 154.9 & 49.5 & 57.9 & 159.0 & 46.9 \\ \hline
		\multicolumn{1}{|l|}{Quantile .95} & 51.9 & 48.8 & 53.9 & 47.0 & 160.4 & 54.9 & 64.3 & 204.2 & 76.6 \\ \hline
	\end{tabular}
		\caption{Realized regret over 2000 trials in Bernoulli experiment}
		\label{table: bernoulli}
\end{table}

\begin{figure}[h!]
\centering
\begin{subfigure}{.49\textwidth}
  \centering
  \includegraphics[width=1\linewidth]{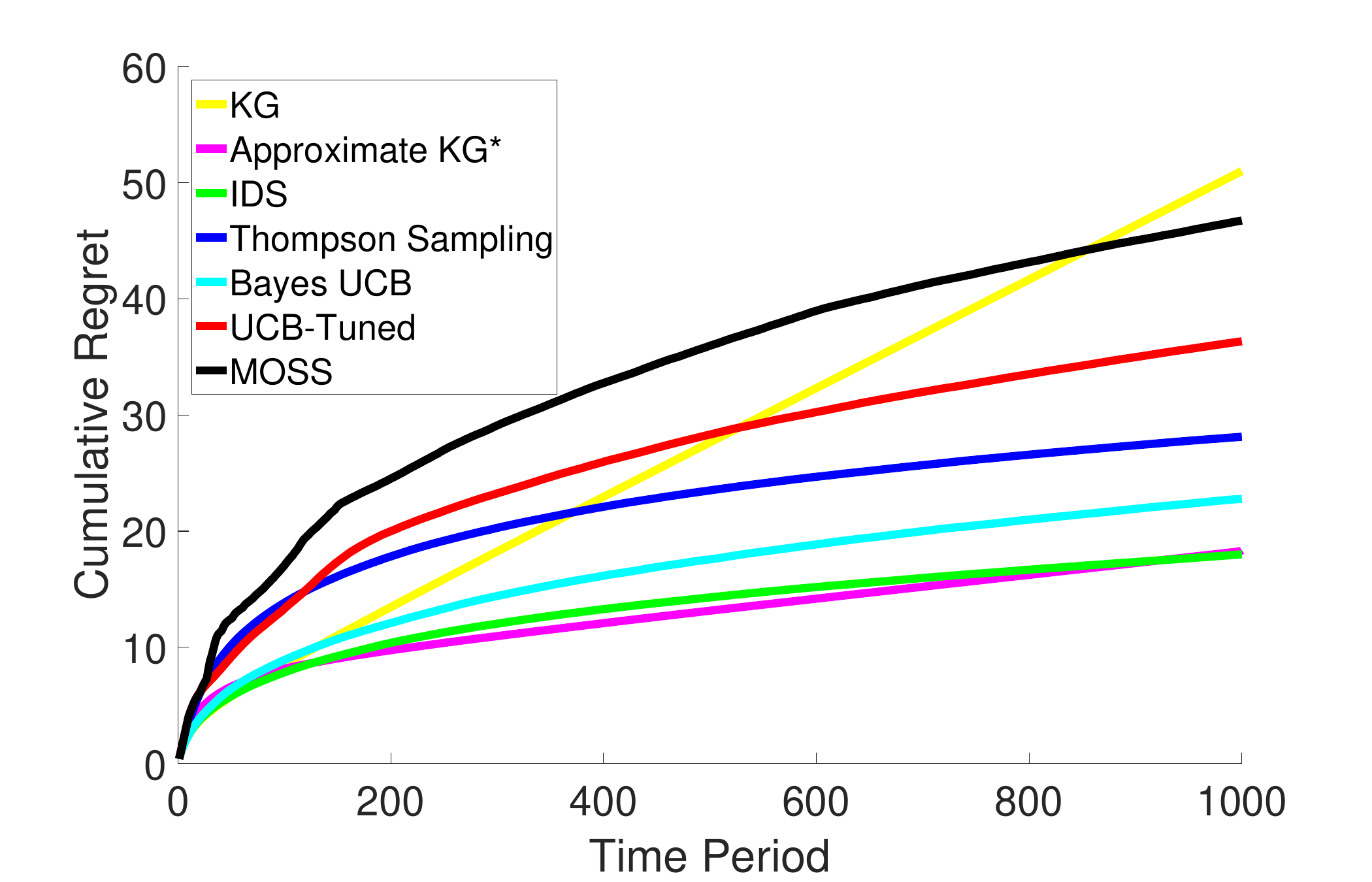}
\caption{Binary rewards}
  \label{fig: bernoulli_regret_plot}
\end{subfigure}
\begin{subfigure}{.49\textwidth}
  \centering
  \includegraphics[width=1\linewidth]{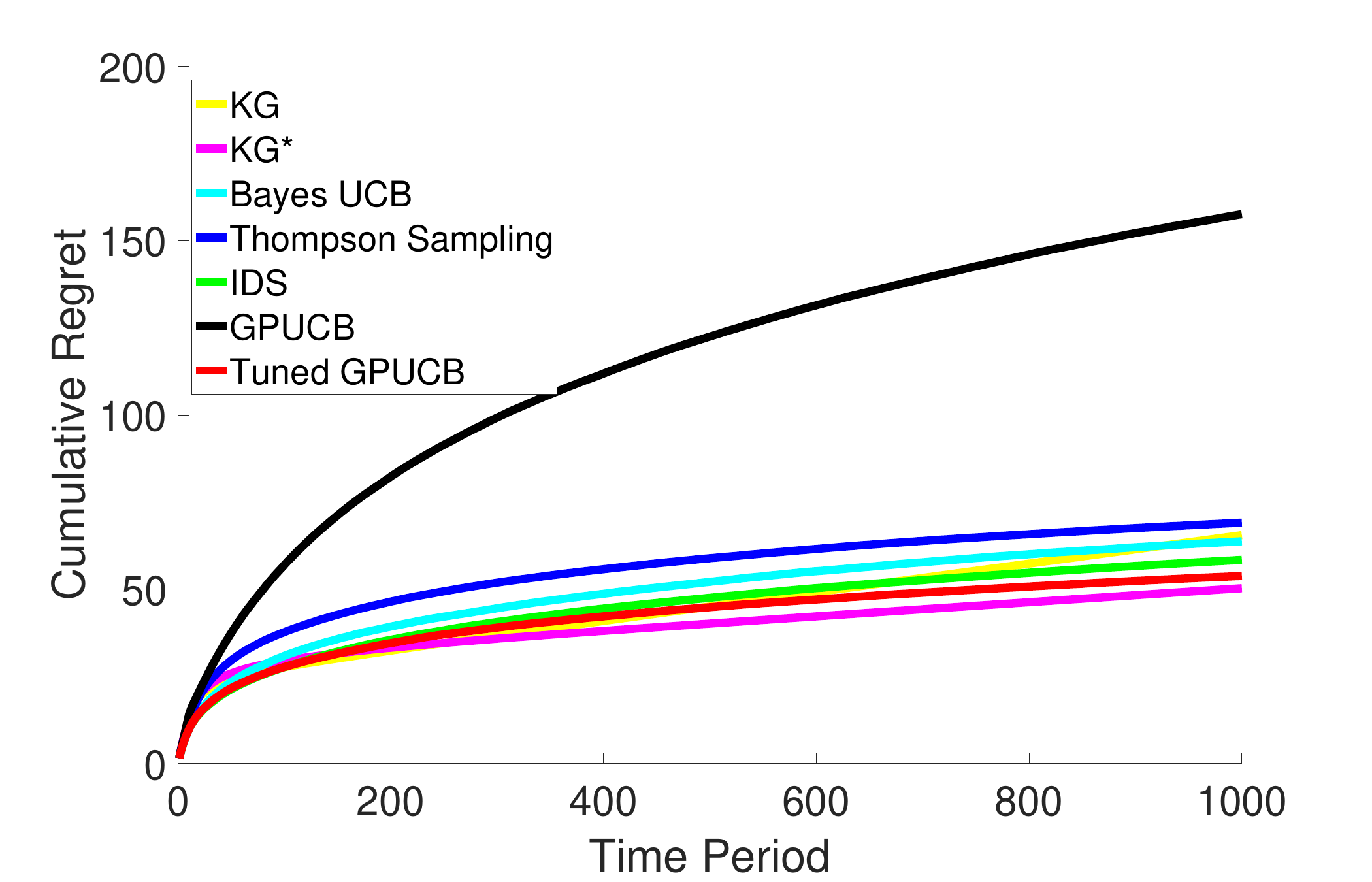}
  \caption{Gaussian rewards}
  \label{fig: gaussian_regret_plot}
\end{subfigure}
\caption{Average cumulative regret over 1000 trials}
\end{figure}


\subsection{Independent Gaussian bandit}\label{subsec: independent normal experiment}

Our second experiment treats a different multi-armed bandit problem with independent arms. The reward value at each action $a$ follows a Gaussian distribution $N( \theta_{a}, 1)$. The mean $\theta_{a}\sim N(0,1)$ is drawn from a Gaussian prior, and the means of different reward distributions are drawn independently. We ran 2000 simulation trials of a problem with 10 arms.  The results are displayed in Figure \ref{fig: gaussian_regret_plot} and Table \ref{table: Gaussian}. 

For this problem, we compare variance-based IDS against Thompson sampling, Bayes UCB, and KG. We use the variance-based variant of IDS because it affords us computational advantages.


\begin{table}[]
	\centering
	\begin{tabular}{l|l|l|l|l|l|l|l|}
		\cline{2-8}
		& \multicolumn{4}{c|}{\textbf{Time Horizon Agnostic}} & \multicolumn{3}{c|}{\textbf{\begin{tabular}[c]{@{}l@{}}Optimized For \\ Time Horizon\end{tabular}}} \\ \hline
		\multicolumn{1}{|l|}{\textbf{Algorithm}} & \textbf{V-IDS} & \textbf{TS} & \textbf{\begin{tabular}[c]{@{}l@{}}Bayes\\ UCB\end{tabular}} & \textbf{GPUCB} & \textbf{\begin{tabular}[c]{@{}l@{}}Tuned\\ GPUCB\end{tabular}} & \textbf{KG} & \textbf{KG*} \\ \hline
		\multicolumn{1}{|l|}{Mean Regret} & 58.4 & 69.1 & 63.8 & 157.6 & 53.8 & 65.5 & 50.3 \\ \hline
		\multicolumn{1}{|l|}{Standard Error} & 1.7 & 0.8 & 0.7 & 0.9 & 1.4 & 2.9 & 1.9 \\ \hline
		\multicolumn{1}{|l|}{Quantile .10} & 24.0 & 39.2 & 34.7 & 108.2 & 24.2 & 16.7 & 19.4 \\ \hline
		\multicolumn{1}{|l|}{Quantile .25} & 30.3 & 47.6 & 43.2 & 130.0 & 30.1 & 20.8 & 24.0 \\ \hline
		\multicolumn{1}{|l|}{Quantile .50} & 39.2 & 61.8 & 57.5 & 156.5 & 41.0 & 25.9 & 29.9 \\ \hline
		\multicolumn{1}{|l|}{Quantile .75} & 56.3 & 80.6 & 76.5 & 184.2 & 58.9 & 36.4 & 40.3 \\ \hline
		\multicolumn{1}{|l|}{Quantile .90} & 104.6 & 104.5 & 97.5 & 207.2 & 86.1 & 155.3 & 74.7 \\ \hline
		\multicolumn{1}{|l|}{Quantile .95} & 158.1 & 126.5 & 116.7 & 222.7 & 112.2 & 283.9 & 155.6 \\ \hline
	\end{tabular}
\caption{Realized regret over 2000 trials  in independent Gaussian experiment}
\label{table: Gaussian}
\end{table}

\begin{table}[]
	\centering
	\begin{tabular}{|l|l|l|l|l|l|l|l|l|l|l|}
		\hline
		\textbf{Time Horizon $T$} & \textbf{10} & \textbf{25} & \textbf{50} & \textbf{75} & \textbf{100} & \textbf{250} & \textbf{500} & \textbf{750} & \textbf{1000} & \textbf{2000} \\ \hline
		\textbf{Regret of V-IDS} & 9.8 & 16.1 & 21.1 & 24.5 & 27.3 & 36.7 & 48.2 & 52.8 & 58.3 & 68.4 \\ \hline
		\textbf{Regret of KG$(T)$} & 9.2 & 15.3 & 20.5 & 22.9 & 25.4 & 35.2 & 45.3 & 52.3 & 62.9 & 80.0 \\ \hline
	\end{tabular}
\caption{ Competitive performance without knowing the time horizon. Average cumulative regret over 2000 trials in the independent Gaussian experiment.}
\label{table: varying time horizon}
\end{table}

We also simulated the GPUCB of \citet{srinivas2012information}. This algorithm maximizes the upper confidence bound $\mu_{t}(a)+\sqrt{\beta_{t}}\sigma_{t}(a)$ where $\mu_{t}(a)$ and $\sigma_{t}(a)$ are the posterior mean and standard deviation of $\theta_{a}$. They provide regret bounds that hold with probability at least $1-\delta$ when $\beta_{t} = 2 \log\left( |\A| t^2 \pi^{2} / 6\delta \right).$ This value of $\beta_{t}$ is far too large for practical performance, at least in this problem setting. The average regret of GPUCB\footnote{We set $\delta =0$ in the definition of $\beta_{t}$, as this choice leads to a lower value of $\beta_{t}$ and stronger performance.} is 157.6, which is roughly almost three times that of V-IDS. For this reason, we considered a tuned version of GPUCB that sets $\beta_{t} = c\log(t)$. We ran 1000 trials of many different values of $c$ to find the value $c=.9$ with the lowest average regret for this problem. This tuned version of GPUCB had average regret of 53.8, which is slight better than IDS.

The work on knowledge gradient (KG) focuses almost entirely on problems with Gaussian reward distributions and Gaussian priors. We find KG performs better in this experiment than it did in the Bernoulli setting, and its average regret is competitive with that of IDS.

As in the Bernoulli setting, KG's realized regret is highly variable. The median regret of KG is the lowest of any algorithm, but in 100 of the 2000 trials its regret exceeded 283  -- seemingly reflecting that the algorithm did not explore enough to identify the best action. The KG* heuristic explores more aggressively, and performs very well in this experiment.  

KG is particularly effective over short time spans. Unlike information-directed sampling, KG takes the time horizon $T$ as an input, and explores less aggressively when there are fewer time periods remaining. Table \ref{table: varying time horizon} compares the regret of KG and IDS over different time horizons. Even though IDS does not take the time horizon into account, 
it is competitive with KG, even over short horizons. We believe that IDS can be modified to exploit fixed and known time horizons more effectively, though we leave the matter for future research.

\subsection{Asymptotic optimality}\label{subsec: asymptotic}
The previous subsections present numerical examples in which IDS outperforms Bayes UCB and Thompson sampling for some problems with independent arms. This is surprising since each of these algorithms is known, in a sense we will soon formalize, to be asymptotically optimal for these problems. This section presents simulation results over a much longer time horizon that suggest IDS scales in the same asymptotically optimal way.

We consider again a problem with binary rewards and independent actions. The action $a_{i} \in \{a_1,\ldots,a_{K}\}$ yields in each time period  a reward that is  1 with probability $\theta_i$ and 0 otherwise. The seminal work of
\citet{lai1985asymptotically} provides the following asymptotic lower bound on regret of any policy $\pi$:
$$\underset{T \rightarrow \infty}{\lim \inf} \, \frac{\E \left[  {\rm Regret}(T, \pi) \vert \theta  \right]}{\log T} \geq \sum_{a\neq A^*} \frac{\theta_{A^*}-\theta_{a} }{ \DKL(\theta_{A^*} \, || \, \theta_{a} ) } := c(\theta).$$
Note that we have conditioned on the parameter vector $\theta$, indicating that this is a frequentist lower bound. Nevertheless, when applied with an independent uniform prior over $\theta$, both Bayes UCB and Thompson sampling are known to attain this lower bound \cite{kaufmann2012bayesian, kaufmann2012thompson}.

Our next numerical experiment fixes a problem with three actions and with $\theta = (.3,.2,.1).$ We compare algorithms over a 10,000 time periods. Due to the expense of running this experiment, we were only able to execute 200 independent trials. Each algorithm uses a uniform prior over $\theta$. Our results, along with the asymptotic lower bound of $c(\theta) \log(T)$, are presented in Figure \ref{fig: lai robbins plot}.
\begin{wrapfigure}{r}{0.5\textwidth}
	\begin{center}
		\includegraphics[width=0.48\textwidth]{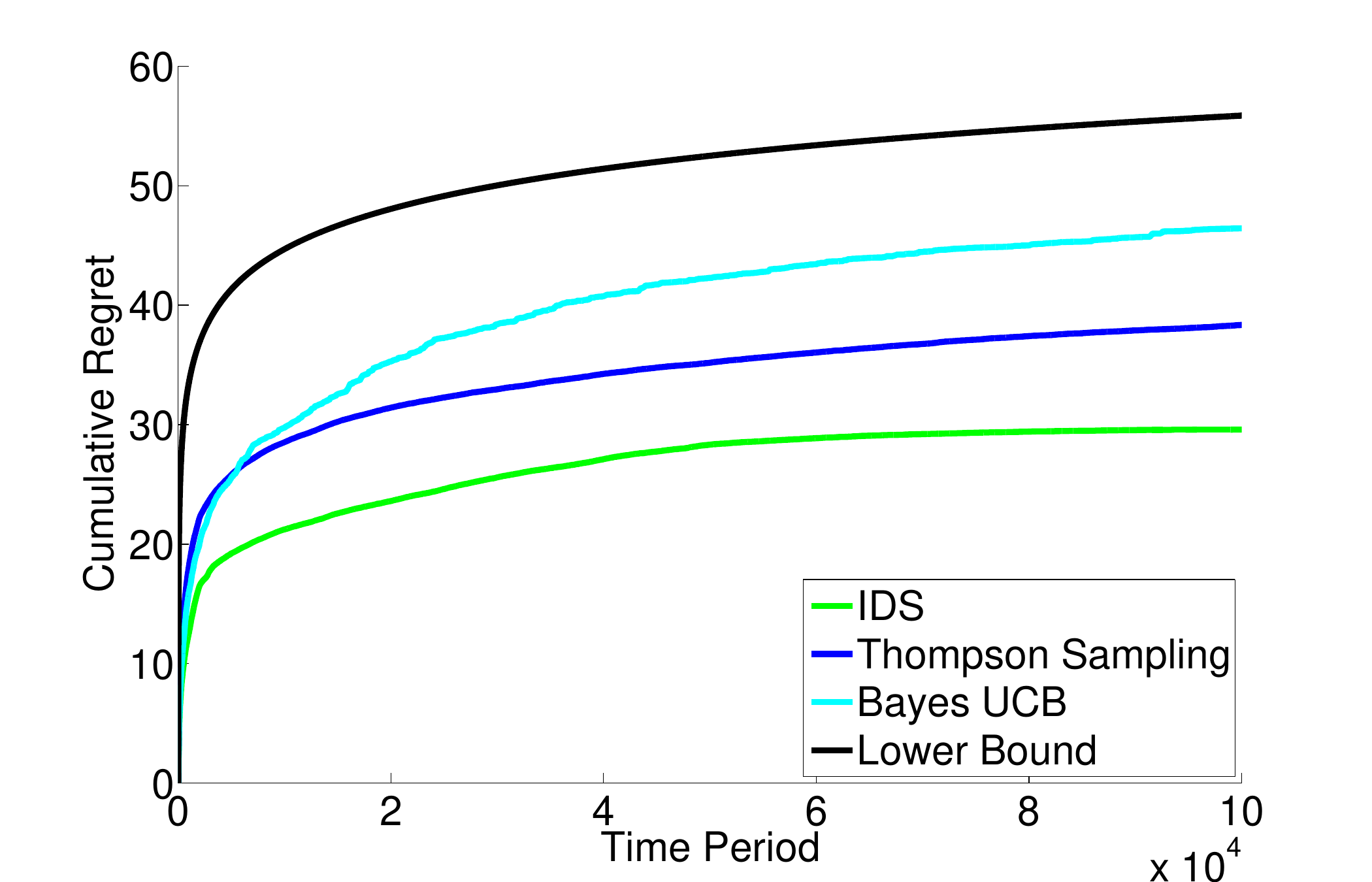}
	\end{center}
	\caption{Cumulative regret over 200 trials.}
	\label{fig: lai robbins plot}
\end{wrapfigure}

\subsection{Linear bandit problems}
\label{subsec: linear experiment}

Our final numerical experiment treats a linear bandit problem. Each action $a\in \mathbb{R}^5$ is defined by a 5 dimensional feature vector. The reward of action $a$ at time $t$ is $a^T \theta + \epsilon_{t}$ where
$\theta \sim N(0, 10I)$ is drawn from a multivariate Gaussian prior distribution, and $\epsilon_{t} \sim N(0,1)$ is independent Gaussian noise. In each period, only the reward of the selected action is observed. In our experiment, the action set $\A$ contains 30 actions, each with features drawn uniformly at random from $[-1/\sqrt{5}, 1/\sqrt{5}]$. The results displayed in Figure \ref{fig: linear} and Table \ref{table: linear KG varying time horizon} compare regret 
across 2,000 independent trials.

We simulate variance-based IDS using the implementation presented in Algorithm \ref{alg: linearIDS}. 
We compare its regret to six competing algorithms. Like IDS, GP-UCB and Thompson sampling satisfy strong regret bounds for this problem\footnote{Regret analysis of GP-UCB can be found in \cite{srinivas2012information}. Regret bounds for Thompson sampling can be found in \cite{agrawal2013linear, russo2014learning, russo2016info}}. Both algorithms are significantly outperformed by IDS. 

We also include Bayes UCB \cite{kaufmann2012bayesian} and a version of GP-UCB that was tuned, as in Subsection \ref{subsec: independent normal experiment},  to minimize its average regret. Each of these displays performance that is competitive with that of  IDS. These algorithms are heuristics, in the sense that the way their confidence bounds are constructed differ significantly from those of linear UCB algorithms that are known to satisfy theoretical guarantees.

\begin{wrapfigure}{r}{0.5\textwidth}
	\begin{center}
		\includegraphics[width=0.43\textwidth]{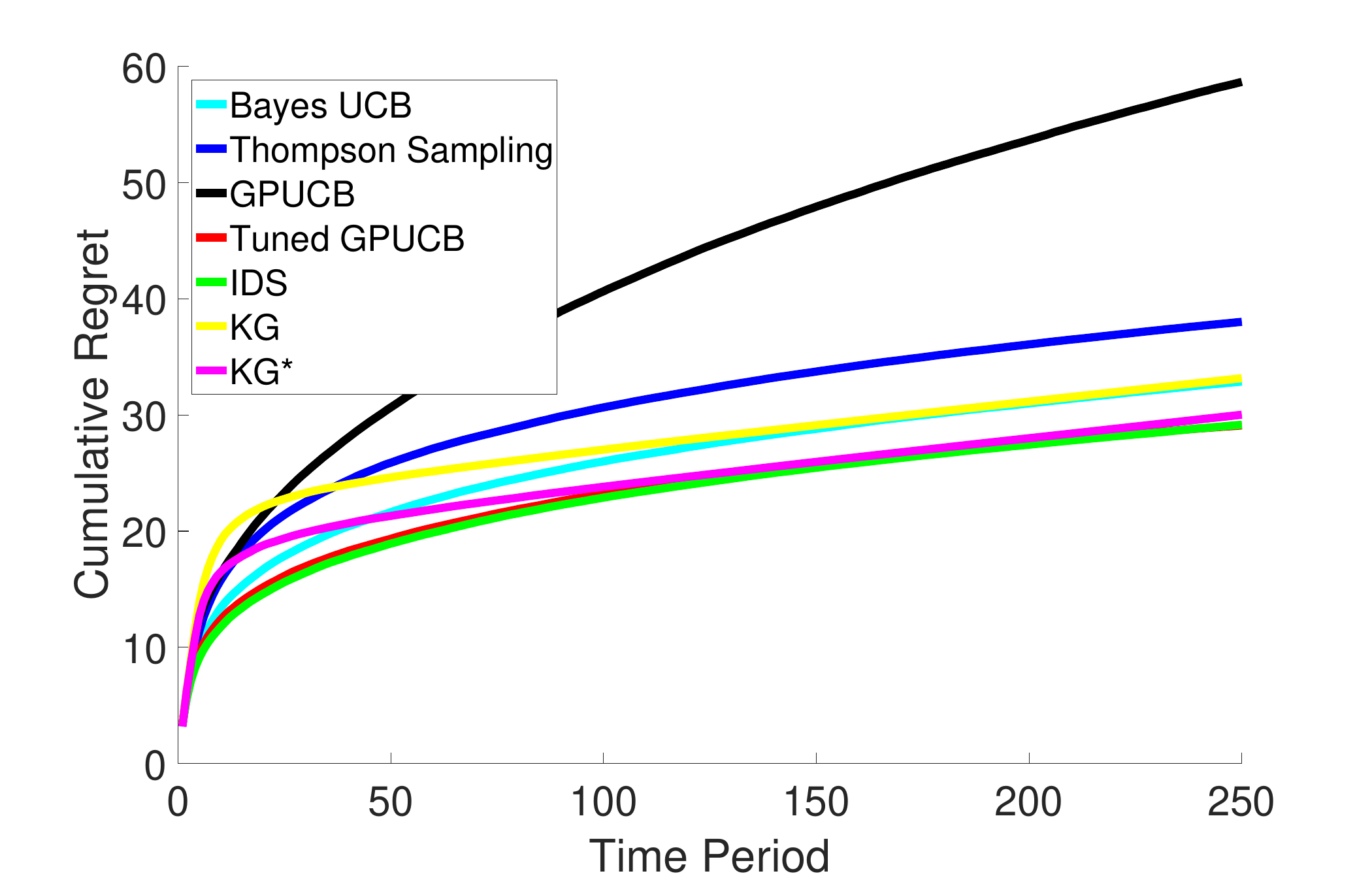}
	\end{center}
	\caption{Regret in linear--Gaussian model.}
	\label{fig: linear}
\end{wrapfigure}
As discussed in Subsection \ref{subsec: independent normal experiment}, unlike IDS, KG takes the time horizon $T$ as an input, and explores less aggressively when there are fewer time periods remaining. Table \ref{table: linear KG varying time horizon} compares IDS to KG over several different time horizons. Even though IDS does not exploit knowledge of the time horizon, it is competitive with KG over short time horizons.

In this experiment, KG* appears to offer a small improvement over standard KG, but as shown in the next subsection, it is much more computationally burdensome. To save computational resources, we have only executed 500 independent trails of the KG* algorithm.

\begin{table}[]
	\centering
	\begin{tabular}{l|l|l|l|l|l|l|l|}
		\cline{2-8}
		& \multicolumn{4}{c|}{\textbf{Time Horizon Agnostic}} & \multicolumn{3}{c|}{\textbf{\begin{tabular}[c]{@{}c@{}}Optimized For \\ Time Horizon\end{tabular}}} \\ \hline
		\multicolumn{1}{|l|}{\textbf{Algorithm}} & \textbf{V-IDS} & \textbf{TS} & \textbf{\begin{tabular}[c]{@{}l@{}}Bayes\\ UCB\end{tabular}} & \textbf{GPUCB} & \textbf{\begin{tabular}[c]{@{}l@{}}Tuned\\ GPUCB\end{tabular}} & \textbf{KG} & \textbf{KG*} \\ \hline
		\multicolumn{1}{|l|}{Mean Regret} & 29.2 & 38.0 & 32.9 & 58.7 & 29.1 & 33.2 & 30.0 \\ \hline
		\multicolumn{1}{|l|}{Standard Error} & 0.5 & 0.4 & 0.4 & 0.3 & 0.4 & 0.7 & 1.4 \\ \hline
		\multicolumn{1}{|l|}{Quantile .10} & 13.0 & 22.6 & 18.9 & 41.3 & 14.5 & 12.7 & 11.9 \\ \hline
		\multicolumn{1}{|l|}{Quantile .25} & 17.6 & 27.6 & 23.1 & 48.9 & 18.4 & 17.5 & 16.1 \\ \hline
		\multicolumn{1}{|l|}{Quantile .50} & 23.2 & 34.3 & 29.2 & 57.9 & 24.0 & 24.1 & 20.6 \\ \hline
		\multicolumn{1}{|l|}{Quantile .75} & 32.1 & 43.7 & 39.0 & 67.4 & 32.9 & 34.5 & 28.5 \\ \hline
		\multicolumn{1}{|l|}{Quantile .90} & 49.5 & 56.5 & 48.7 & 77.1 & 46.6 & 60.9 & 55.6 \\ \hline
		\multicolumn{1}{|l|}{Quantile .95} & 67.5 & 67.5 & 58.4 & 82.7 & 59.9 & 94.5 & 96.1 \\ \hline
	\end{tabular}
	\caption{Realized regret over 2000 trials  in linear experiment. KG* results are over 500 trails.}
	\label{my-label}
\end{table}

\begin{table}[]
	\centering
	
	\begin{tabular}{|l|l|l|l|l|l|l|l|}
		\hline
		\textbf{Time Horizon $T$} & \textbf{10} & \textbf{25} & \textbf{50} & \textbf{75} & \textbf{100} & \textbf{250} & \textbf{500} \\ \hline
		\textbf{Regret of V-IDS} & 11.8 & 16.2 & 19.6 & 21.6 & 23.3 & 31.1 & 34.7 \\ \hline
		\textbf{Regret of KG$(T)$} & 11.1 & 15.1 & 19.0 & 22.5 & 24.1 & 34.4 & 43.0 \\ \hline
	\end{tabular}
	\caption{Competitive performance without knowing the time horizon Average cumulative regret over 2000 trials in linear Gaussian experiment.}
	\label{table: linear KG varying time horizon}
\end{table}

\subsection{Runtime Comparison}
We now compare the time required to compute decisions using the algorithms we have applied.
In our experiments, Thompson sampling and UCB algorithms are extremely fast, sometimes requiring only a few microseconds to reach a decision. As expected, our implementation of IDS requires significantly more compute time. However, IDS often reaches a decision in only a small fraction of second, which is tolerable in many application areas. In addition, IDS may be accelerated considerably via parallel processing or an optimized implementation. 

The results for KG are mixed. For independent Gaussian models,  certain integrals can be computed via closed form expressions, allowing KG to execute quickly. There is also a specialized numerical procedure for implementing KG for correlated (or linear)  Gaussian models, but computation is an order of magnitude slower than in the independent case. For correlated Gaussian models, the KG* policy is much slower than both KG and IDS. For beta-Bernoulli problems, KG can be computed very easily, but yields poor performance. A direct implementation of the KG* policy was too slow to simulate, and so we have used a heuristic approach presented in \cite{kaminski2015refined}, which uses golden section search to maximize a function that is not necessarily unimodal. This method is labeled ``Approx KG*'' in Table \ref{table: Bernoulli time}.

Table \ref{table: Bernoulli time} displays results for the Bernoulli experiment described in Subsection \ref{subsec: Bernoulli experiment}. It shows the average time required to compute a decision in a 1000 period problem with $10, 30,50$ and $70$ arms. IDS was implementing using Algorithm  \ref{algo: beta} to evaluate the information ratio, and Algorithm \ref{alg: chooseAction}  to optimize it. The numerical integrals in Algorithm \ref{algo: beta} were approximated using quadrature with 1000 equally spaced points. Table \ref{table: Gaussian time} presents results of the corresponding experiment in the Gaussian case. Finally, Table \ref{table: linear time} displays results for the linear bandit experiments described in Subsection \ref{subsec: linear experiment}, which make use of Algorithm 6 and Markov chain Monte Carlo sampling with $M=10,000$ samples.  The table provides the average time required to compute a decision in a 250 period problem.

\begin{table}[]
	\centering
	\begin{tabular}{|l|l|l|l|l|l|l|l|}
		\hline
		\textbf{Arms} & \textbf{IDS} & \textbf{V-IDS} & \textbf{TS} & \textbf{Bayes UCB} & \textbf{UCB1} & \textbf{KG} & \textbf{Approx KG*} \\ \hline
		10 & 0.011013 & 0.01059 & 0.000025 & 0.000126 & 0.000008 & 0.000036 & 0.074618 \\ \hline
		30 & 0.047021 & 0.047529 & 0.000023 & 0.000147 & 0.000005 & 0.000017 & 0.215145 \\ \hline
		50 & 0.104328 & 0.10203 & 0.000024 & 0.000176 & 0.000005 & 0.000017 & 0.358505 \\ \hline
		70 & 0.18556 & 0.178689 & 0.000028 & 0.000167 & 0.000005 & 0.000017 & 0.494455 \\ \hline
	\end{tabular}
		\caption{Bernoulli Experiment: Compute time per-decision in seconds.}
		\label{table: Bernoulli time}
\end{table}

\begin{table}[]
	\centering
	\begin{tabular}{|l|l|l|l|l|l|l|}
		\hline
		\textbf{Arms} & \textbf{V-IDS} & \textbf{TS} & \textbf{Bayes UCB} & \textbf{GPUCB} & \textbf{KG} & \textbf{KG*} \\ \hline
		10 & 0.00298 & 0.000008 & 0.00002 & 0.00001 & 0.000146 & 0.001188 \\ \hline
		30 & 0.012597 & 0.000005 & 0.000009 & 0.000005 & 0.000097 & 0.003157 \\ \hline
		50 & 0.023084 & 0.000006 & 0.000009 & 0.000005 & 0.000094 & 0.005146 \\ \hline
		70 & 0.03913 & 0.000006 & 0.000009 & 0.000005 & 0.000098 & 0.006364 \\ \hline
	\end{tabular}
	\caption{Independent Gaussian Experiment: Compute time per-decision in seconds.}
	\label{table: Gaussian time}
\end{table}

\begin{table}[]
	\centering
	\begin{tabular}{|l|l|l|l|l|l|l|l|}
		\hline
		\textbf{Arms} & \textbf{Dimension} & \textbf{V-IDS} & \textbf{TS} & \textbf{Bayes UCB} & \textbf{GPUCB} & \textbf{KG} & \textbf{KG*} \\ \hline
		15 & 3 & 0.004305 & 0.000178 & 0.000139 & 0.000048 & 0.002709 & 0.311935 \\ \hline
		30 & 5 & 0.008635 & 0.000064 & 0.000048 & 0.000038 & 0.004789 & 0.589998 \\ \hline
		50 & 20 & 0.026222 & 0.000077 & 0.000083 & 0.000068 & 0.008356 & 1.051552 \\ \hline
		100 & 30 & 0.079659 & 0.000115 & 0.000148 & 0.00013 & 0.017034 & 2.067123 \\ \hline
	\end{tabular}
	\caption{Linear Gaussian Experiment: Compute time per-decision in seconds.}
	\label{table: linear time}
\end{table}

\section{Conclusion}
This paper has proposed information-directed sampling  -- a new algorithm for online optimization problems in which a decision maker must learn from partial feedback. We establish a general regret bound for the algorithm, and specialize this bound to several widely studied problem classes. We show that it sometimes greatly outperforms other popular approaches, which don't carefully measure the information provided by sampling actions. Finally, for some simple and widely studied classes of multi-armed bandit problems we demonstrate simulation performance surpassing popular approaches.

Many important open questions remain, however. IDS solves a single-period optimization problem as a proxy to an intractable multi-period problem.  Solution of this single-period problem can itself be computationally demanding, especially in cases where the number of actions is enormous or mutual information is difficult to evaluate.  An important direction for future research concerns the development of computationally elegant procedures to implement IDS in important cases. Even when the algorithm cannot be directly implemented, however, one may hope to develop simple algorithms that capture its main benefits. Proposition \ref{prop: regret bound average information ratio} shows that any algorithm with small information ratio satisfies strong regret bounds. Thompson sampling is a simple algorithm that, we conjecture, sometimes has nearly minimal information ratio. Perhaps simple schemes with small information ratio could be developed for other important problem classes, like the sparse linear bandit problem.

In addition to computational considerations, a number of statistical questions remain open. One question raised is whether IDS attains the lower bound of \citet{lai1985asymptotically} for some bandit problems with independent arms. Beyond the empirical evidence presented in Subsection \ref{subsec: asymptotic}, there are some theoretical reasons to conjecture this is true. Next, a more precise understanding of problem's {\it information complexity} remains an important open question for the field. Our regret bound depends on the problem's information complexity through a term we call the information ratio, but it's unclear if or when this is the right measure. Finally, it may be possible to derive lower bounds using the same information theoretic style of argument used in the derivation of our upper bounds.

\section{Extensions}
This section presents a  number of ways in which the results and ideas discussed throughout this paper can be extended. We will consider the use of algorithms like information-directed sampling for pure--exploration problems, a form of information-directed sampling that aims to acquire information about $\theta$ instead of $A^*$, and a version of information directed-sampling that uses a tuning parameter to control how aggressively the algorithm explores. In each case, new theoretical guarantees can be easily established by leveraging our analysis of information-directed sampling.

\subsection{Pure exploration problems}\label{subsec: pure exploration}

Consider the problem of adaptively gathering observations $\left( A_{1}, Y_{1, A_{1}},\ldots, A_{T-1}, Y_{T-1,A_{T-1}} \right)$ so as to minimize the expected loss of the best decision at time $T$,
\begin{equation}\label{eq: terminal loss}
\E \left[  \min_{a\in \A} \Delta_{T}(a) \right].
\end{equation}
Recall that we have defined $\Delta_{t}(a):= \E \left[R_{t,A^*} -  R_{t,a}   \vert \hist \right]$ to be the expected regret of action $a$ at time $t$.  This is a ``pure exploration problem,'' in the sense that one is interested only in the terminal regret \eqref{eq: terminal loss} and not in the algorithm's cumulative regret.  However, the next proposition shows that bounds on the algorithm's cumulative expected regret imply bounds on $\E \left[  \min_{a\in \A} \Delta_{T}(a) \right]$.

\begin{prop} \label{prop: terminal regret from cumulative regret}
	If actions are selected according to a policy $\pi$, then
	\begin{equation*}
	\E \left[  \min_{a\in \A} \Delta_{T}(a) \right]  \leq \frac{\E \left[ {\rm Regret}\left( T, \pi \right)  \right]}{T}.
	\end{equation*}
	\begin{proof}
		By the tower property of conditional expectation,
		$\E \left[ \Delta_{t+1}(a) | \hist \right] = \Delta_{t}(a)$. Therefore, Jensen's inequality shows
		$\E \left[ \min_{a \in \A} \Delta_{t+1}(a) | \hist \right] \leq \min_{a\in \A} \Delta_{t}(a) \leq \Delta_{t}(\pi_t) $. Taking expectations and iterating this relation shows that
		\begin{equation}\label{eq: more info is better}
		\E \left[ \min_{a \in \A} \Delta_{T}(a) \right] \leq \E \left[ \min_{a \in \A} \Delta_{t}(a) \right] \leq  \E \left[ \Delta_{t}(\pi_{t}) \right] \hspace{10pt} \forall t \in \{1,\ldots,T\}.
		\end{equation}
		The result follows by summing both sides of \eqref{eq: more info is better} over $t \in \{1,\ldots,T\}$ and dividing each by $T$.
	\end{proof}
\end{prop}

Information-directed sampling is  designed to have low cumulative regret, and therefore balances between acquiring information and taking actions with low expected regret. For pure exploration problems, it's natural instead to consider an algorithm that always acquires as much information about $A^*$ as possible. The next proposition provides a theoretical guarantee for an algorithm of this form. The proof of this result combines our analysis of information-directed sampling with Proposition \ref{prop: terminal regret from cumulative regret}. 

\begin{prop}\label{prop: pure exploration ids}
	If actions are selected so that
	$$A_{t} \in \underset{a \in \A}{\arg \max} \, g_t(a),$$
	and $\Psi_{t}^* \leq \lambda$ almost surely for each $t\in \{1,\ldots,T  \}$, then
	$$\E \left[  \min_{a\in \A} \Delta_{T}(a) \right] \leq \sqrt{\frac{\lambda H(\alpha_{1}) }{T}}.$$
\end{prop}
\begin{proof}
	To simplify notation, let $\Delta_{t}^* = \min_{a\in \A} \Delta_{t}(a)$ denote the minimal expected regret at time $t$, and $g_{t}^* = \max_{a\in \A} g_{t}(a)$  denote the information gain under the current algorithm. 
	
	Since $\Delta_{t}(\pi_{t}^{\rm IDS})^2 \leq \lambda g_{t}(\pi_{t}^{\rm IDS})$, it is immediate that $\Delta_{t}^* \leq \sqrt{\lambda g_t^*}$. Therefore 
	\[
	\E[\Delta_{T}^*] \overset{(a)}{\leq} \left(\frac{1}{T}\right)\E \sum_{t=1}^{T}\Delta^*_{t} \leq \left(\frac{\sqrt{\lambda}}{T}\right)\E \sum_{t=1}^{T} \sqrt{g_t^*} \overset{(b)}{\leq} \left(\frac{\sqrt{\lambda}}{T}\right)\sqrt{T\E \sum_{t=1}^{T} g^*_t} \overset{(c)}{\leq} \sqrt{\frac{\lambda H(\alpha_1)}{T}}.
	\] 
	Inequality (a) uses equation \eqref{eq: more info is better} in the proof of  Proposition \ref{prop: terminal regret from cumulative regret}, (b) uses the Cauchy-Schwartz inequality, and (c) follows as in the proof of Proposition \ref{prop: regret bound average information ratio}. 
\end{proof}

\subsection{Using information gain about $\theta$}\label{subsec: theta-IDS}
Information-directed sampling optimizes a single-period objective that balances earning high immediate reward and acquiring information. Information is quantified using the mutual information between the true optimal action $A^*$ and the algorithm's next observation $Y_{t,a}$. In this subsection, we will consider an algorithm that instead quantifies the amount learned through selecting an action $a$ using the mutual information $ I_{t}\left( \theta ; Y_{t,a} \right)$ between the algorithm's next observation and the unknown parameter $\theta$. As highlighted in Subsection \ref{subsec: other info directed}, such an algorithm could invest in acquiring information that is irrelevant to the decision problem. However, in some cases, such an algorithm can be computationally simple while offering reasonable statistically efficiency.

We introduce a modified form of the information ratio
\begin{equation}
\Psi^{\theta}_{t}(\pi) : =\frac{ \Delta_{t}(\pi)^2 } { \sum_{a\in\A} \pi(a) I_{t}\left( \theta ; Y_{t,a} \right)}
\end{equation}
which replaces the expected information gain about $A^*$, $g_{t}(\pi)=  \sum_{a\in\A} \pi(a) I_{t}\left( A^* ; Y_{t,a}\right)$, with the expected information gain about $\theta$.

\begin{prop}
	For any action sampling distribution $\tilde{\pi} \in \D( \A)$,
	\begin{equation}\label{eq: info ratio about model is smaller}
	\Psi^{\theta}_{t}(\tilde{\pi}) \leq  \Psi_{t}(\tilde{\pi}).
	\end{equation}
	Furthermore, if $\Theta$ is finite, and there is some $\lambda \in \mathbb{R}$ and policy $\pi = (\pi_{1}, \pi_{2}, \ldots)$ satisfying
	$ \Psi^{\theta}_{t}(\pi_t) \leq \lambda $ almost surely, then
	\begin{equation}\label{eq: regret for model based IDS}
	\E\left[ {\rm Regret}(T, \pi)  \right] \leq \sqrt{ \lambda H(\theta)  T}.
	\end{equation}
\end{prop}
Equation \eqref{eq: info ratio about model is smaller} relies on the inequality  $I_{t}\left( A^* ; Y_{t,a} \right) \leq I_{t}\left( \theta ; Y_{t,a} \right)$, which itself follows from the data processing inequality of mutual information because $A^*$ is a function of $\theta$. The proof of the second part of the proposition is almost identical to the proof of Proposition \ref{prop: regret bound average information ratio}, and is omitted.

We have provided several bounds on the information ratio of $\pi^{\rm IDS}$ of the form $\Psi_{t}(\pi^{\rm IDS}_{t}) \leq \lambda$. By this proposition, such bounds imply that if $\pi=(\pi_1, \pi_2, \ldots)$ satisfies
$$ \pi_{t} \in \underset{\pi \in \D(\A)}{ \arg\min }\Psi^{\theta}_{t}(\pi) $$
then, $\Psi^{\theta}_{t}(\pi_{t}) \leq \Psi^{\theta}_{t}(\pi_{t}^{\rm IDS}) \leq \Psi_{t}(\pi_{t}^{\rm IDS}) \leq \lambda$, and the regret bound \eqref{eq: regret for model based IDS} applies.

\subsection{A tunable version of information-directed sampling}
In this section, we present an alternative form of information-directed sampling that depends on a tuning parameter $\lambda \in \mathbb{R}$. As $\lambda$ varies, the algorithm strikes a different balance between exploration and exploration. The following proposition provides regret bounds for this algorithm provided $\lambda$ is sufficiently large.
\begin{prop}
	Fix any $\lambda \in \mathbb{R}$ such that $\Psi_{t}(\pi_{t}^{\rm IDS}) \leq \lambda$ almost surely for each $t\in \{1,\ldots,T \}$. If $\pi = (\pi_{1}, \pi_{2},..)$ is defined so that
	\begin{equation}\label{eq: tunable IDS}
	\pi_{t} \in   \underset{\pi \in \D(\A)}{\arg \min}  \left\{\, \rho(\pi):= \Delta_{t}(\pi)^2 - \lambda g_{t}(\pi) \right\},
	\end{equation}
	then
	$$\E \left[{\rm Regret}(T, \pi)\right] \leq \sqrt{ \lambda H(\alpha) T}.$$
\end{prop}
\begin{proof}
	We have that
	$$\rho(\pi_{t}) \overset{(a)}{\leq} \rho(\pi_{t}^{\rm IDS})  \overset{(b)}{\leq} 0,$$
	where (a) follows since $\pi_{t}^{\rm IDS}$ is feasible for the optimization problem \eqref{eq: tunable IDS}, and (b) follows since $$0=\Delta_{t}(\pi_{t}^{\rm IDS})^2 - \Psi_{t}(\pi_{t}^{\rm IDS})  g_{t}(\pi_{t}^{\rm IDS}) \geq \Delta_{t}(\pi_{t}^{\rm IDS})^2  - \lambda g_{t}(\pi_{t}^{\rm IDS}). $$
	Since $\rho_{t}(\pi_{t}) \leq 0$, it must be the case that
	$\lambda \geq \Delta_{t}(\pi_{t})^2 / g_{t}(\pi_{t})  \overset{\rm Def}{=} \Psi_{t}( \pi _{t})$.  The result then follows by applying Proposition \ref{cor: general regret bound}.
\end{proof}

\subsection*{Acknowledgements}

We thank the anonymous referees for feedback and stimulating exchanges, Junyang Qian for correcting miscalculations in an earlier draft pertaining to Examples \ref{ex:sparse} and 
\ref{ex:assortment}, and Eli Gutin and Yashodan Kanoria for helpful suggestions.  This work was generously supported by a research grant from Boeing, a Marketing Research Award from Adobe, and the Burt and Deedee McMurty Stanford Graduate Fellowship.

\begin{appendix}

\section{Proof of Proposition \ref{prop: support at most 2}}
\begin{propn}[\ref{prop: support at most 2}]
For all $\vec{\Delta}, \vec{g} \in \mathbb{R}^K_+$ such that $\vec{g} \neq 0$,
the function $\pi \mapsto \left(\pi^\top \vec{\Delta}\right)^{2}/ \pi^\top \vec{g}$ is convex on $\left\{\pi \in \mathbb{R}^K : \pi^\top \vec{g} > 0  \right\}$.
Moreover, this function is minimized over $\mathcal{S}_K$ by some $\pi^*$ for which $|\{k : \pi^*_k >0 \}| \leq 2$.
\end{propn}

\begin{proof}
First, we show the function $\Psi: \pi \mapsto \left(\pi^{T}\Delta\right)^{2}/ \pi^{T}g$ is convex on $\left\{\pi \in \mathbb{R}^{K} \vert \pi^T g > 0  \right\}$. As shown in Chapter 3 of \citet{boyd2004convex}, $f: (x,y) \mapsto x^2/ y$ is convex over $\{(x,y) \in \mathbb{R}^2: y>0\}$. The function $h: \pi  \mapsto (\pi^{T}\Delta, \pi^{T} g) \in \mathbb{R}^{2}$  is affine. Since convexity is preserved under composition with an affine function, the function $\Psi=g\circ h$ is convex.

We now prove the second claim. Consider the optimization problems
\begin{eqnarray} \label{eq: psi minimization problem}
\text{minimize } \Psi(\pi) \text{ subject to   } \pi^T e =1, \pi\geq 0 \\ \label{eq: rho minimization problem}
\text{minimize } \rho(\pi) \text{ subject to   } \pi^T e =1, \pi\geq 0
\end{eqnarray}
where
$$\rho(\pi):=\left( \pi^T \Delta \right)^2-  \left(\pi^T g  \right)\Psi^* ,$$
and  $\Psi^*\in \mathbb{R}$ denotes the optimal objective value for the minimization problem \eqref{eq: psi minimization problem}. The set of optimal solutions to \eqref{eq: psi minimization problem} and \eqref{eq: rho minimization problem} correspond. Note that
$$ \Psi(\pi) = \Psi^* \implies \rho(\pi) =0$$
but for any feasible $\pi$,  $\rho(\pi) \geq 0$ since $\Delta(\pi)^2 \geq \Psi^* g(\pi)$. Therefore, any optimal solution $\pi_0$ to \eqref{eq: psi minimization problem} is an optimal solutions to \eqref{eq: rho minimization problem} and satisfies $\rho(\pi_0) = 0$. Similarly, if $\rho(\pi) = 0$ then simple algebra shows that $\Psi(\pi) = \Psi^*$ and hence that $\pi$ is an optimal solution to \eqref{eq: psi minimization problem}

We will now show that there is a minimizer of $\rho(\cdot)$ with at most two nonzero components, which implies the same is true of $\Psi(\cdot)$.
%
Fix a minimizer $\pi^*$ of $\rho(\cdot)$. Differentiating $\rho(\pi)$ with respect to $\pi$ at $\pi=\pi^*$ yields
\begin{eqnarray*}\frac{\partial}{\partial \pi} \rho(\pi^*) &=&2\left( \Delta^T \pi^* \right) \Delta -\Psi^* g\\
&=&2 L^* \Delta- \Psi^* g
\end{eqnarray*}
where $L^* = \Delta^T \pi^*$ is the expected instantaneous regret of the sampling distribution $\pi^*$. Let $d^* = \min_{i} \frac{\partial}{\partial \pi_{i}} \rho(\pi^*)$ denote the smallest partial derivative of $\rho$ at $\pi^*$. It must be the case that any $i$ with $\pi_{i}^*>0$ satisfies  $d^* = \frac{\partial}{\partial \pi_{i}} \rho(\pi^*)$, as otherwise transferring probability from action $a_i$ could lead to strictly lower cost. This shows that
\begin{equation}\label{eq: linear first order condition}
\pi_i^* >0 \implies g_{i} = \frac{-d^*}{\Psi^*} + \frac{2L^*}{\Psi^*} \Delta_{i}.
\end{equation}
Let $i_{1},..,i_{m}$ be the indices such that  $\pi_{i_k}^*>0$ ordered so that
$g_{i_1} \geq g_{i_2} \geq \cdots \geq g_{i_m}.$
Then we can choose a $\beta \in [0,1]$ so that
$$\sum_{k=1}^{m} \pi_{i_k}^* g_{i_k}  = \beta g_{i_1} +(1-\beta) g_{i_m}.$$
By equation \eqref{eq: linear first order condition}, this implies as well that $\sum_{k=1}^{m} \pi_{i_k}^* \Delta_{i_k}  = \beta \Delta_{i_1} +(1-\beta) \Delta_{i_m},$  and hence that the sampling distribution that plays $a_{i_1}$ with probability $\beta$ and $a_{i_m}$ otherwise has the same instantaneous expected regret and the same expected information gain as $\pi^*$. That is, starting with a general sampling distribution $\pi^*$ that maximizes $\rho(\pi)$, we showed there is a sampling distribution with support over at most two actions attains the same objective value and hence that also maximizes $\rho(\pi)$.
\end{proof}

\section{Proof of Proposition  \ref{prop: regret bound average information ratio}}
The following fact expresses the mutual information between $A^*$ and $Y_{t,a}$ as the  as the expected reduction in the entropy of $A^*$ due to observing $Y_{t,a}$.
\begin{fact}\label{fact: mutual information to entropy}
(Lemma 5.5.6 of \citet{gray2011entropy}) $$I_t \left( A^* ; Y_{t,a} \right)  = \E\left[ H(\alpha_{t})-H(\alpha_{t+1}) \vert  A_{t}=a,  \hist \right]$$
\end{fact}

\begin{propn}[\ref{prop: regret bound average information ratio}]
 For any policy $\pi=(\pi_1, \pi_2, \pi_3, \ldots)$ and time $T\in \mathbb{N}$,
$$ \E \left[\mathrm{Regret}\left(T, \pi \right) \right] \leq \sqrt{ \overline{\Psi}_{T}(\pi) H(\alpha_{1}) T }.$$
where
\[
\overline{\Psi}_{T}(\pi) \equiv \frac{1}{T} \sum_{t=1}^{T} \E_{\pi} [  \Psi_{t}(\pi_t ) ]
\]
is the average expected information ratio under $\pi$.
\end{propn}
\begin{proof}
Since the policy $\pi$ is fixed throughout, we will simplify notation and write $\Psi_{t} \equiv \Psi_{t}(\pi_t)$, $\Delta_{t}\equiv \Delta_{t}(\pi_t)$ and $g_{t}=g_{t}(\pi_t)$ throughout this proof.  First observe that entropy bounds  expected cumulative information gain: 
$$\E \sum_{t=1}^{T} g_{t}= \E \sum_{t=1}^{T} \E  \left[ H(\alpha_{t}) - H(\alpha_{t+1}) \vert \hist \right]
= \E \sum_{t=1}^{T} \left( H(\alpha_{t}) - H(\alpha_{t+1}) \right)= H(\alpha_{1}) - H(\alpha_{T+1}) \leq H(\alpha_{1}),$$
where the first equality relies on Fact \ref{fact: mutual information to entropy} and the tower property of conditional expectation and the final inequality follows from the non-negativity of entropy. Then, 
\begin{eqnarray*}
\E\left[ {\rm Regret}\left(T, \pi \right) \right] =  \E \sum_{t=1}^{T}  \Delta_t  = \E \sum_{t=1}^{T} \sqrt{\Psi_{t}} \sqrt{g_{t}\left(\pi_{t}^{{\rm IDS}}\right) } 
& \leq & \sqrt{\E  \sum_{t=1}^{T} \Psi_{t} } \sqrt{\E \sum_{t=1}^{T} g_{t}}\\
& \leq & \sqrt{ H(\alpha_1)}  \sqrt{\E  \sum_{t=1}^{T} \Psi_{t} }  \\
& = &  \sqrt{\left(\frac{1}{T} \E  \sum_{t=1}^{T} \Psi_{t} \right) H(\alpha_1) T},
\end{eqnarray*}
where the first inequality follows from Holder's inequality. 
\end{proof}

\section{Proof of Proposition \ref{prop: information ratio bounds for v-ids}}
\begin{propn}[\ref{prop: information ratio bounds for v-ids}]
Suppose $\sup_{y} R(y) - \inf_{y} R(y) \leq 1$  and
\[
\pi_{t} \in \arg \min _{\pi \in \mathcal{S}_K} \frac{\Delta_{t}(\pi)^2}{v_{t}(\pi)},
\] 
Then the following hold:
\begin{enumerate}
\item $\Psi_{t}(\pi_t) \leq |\A|/2$.
\item $\Psi_{t}(\pi_t) \leq d/2$ when  $\A \subset \mathbb{R}^d$, $\Theta \subset \mathbb{R}^d$,  and
$\E\left[ R_{t,a}| \theta  \right] = a^T \theta$ for each action $a\in \A$.
\end{enumerate}
\end{propn}
The proof of this proposition essentially reduces to techniques in \citet{russo2016info}, but some new analysis is required to show the results in that paper apply to variance-based IDS. A full proof is provided below. 

We will make use of the following fact, which is a matrix-analogue of the Cauchy-Schwartz inequality. For any rank $r$ matrix $M \in \mathbb{R}^{n\times n}$ with singular values $\sigma_{1},\ldots,\sigma_{r}$,  let
\begin{eqnarray*}
\| M\|_{*} : = \sum_{i=1}^{r} \sigma_{i}, \hspace{6pt}&\hspace{6pt}  \| M \|_F := \sqrt{ \sum_{k=1}^{n} \sum_{j=1}^{n} M_{i,j}^2} = \sqrt{\sum_{i=1}^{r} \sigma_{i}^2},  \hspace{6pt}&\hspace{6pt}  {\rm Trace}(M):=\sum_{i=1}^{n} M_{ii},
\end{eqnarray*}
denote respectively the Nuclear norm, Frobenius norm and trace of $M$.
\begin{fact}\label{fact: trace frobenius inequality}
For any matrix $M \in \mathbb{R}^{k\times k}$, $${\rm Trace}\left( M \right) \leq \sqrt{{\rm Rank}(M)}\| M\|_{\rm F}.$$
\end{fact}

We now prove Proposition \ref{prop: information ratio bounds for v-ids}
\begin{proof}$\,$\\
{\bf Preliminaries:} 
As noted in Section \ref{subsec: approximating the information ratio}, $g_{t}(a) \geq 2v_{t}(a)$ for all $t$ and $a$. Therefore for any $\pi \in \D(\A)$
\[ 
\Psi_{t}(\pi) = \frac{\Delta_{t}(\pi)^2}{g_{t}(\pi)} \leq \frac{\Delta_{t}(\pi)^2}{2v_{t}(\pi)}.
\] 
Therefore, if 
\[ 
\pi_{t} = \underset{\pi \in \D(\A)}{\arg\min}\frac{\Delta_{t}(\pi)^2}{v_{t}(\pi)}
\] 
is the action-sampling distribution chosen by variance based IDS, then 
\[ 
\Psi_{t}(\pi_t) \leq \frac{\Delta_{t}(\pi_t)^2}{2v_{t}(\pi_t)} \leq \frac{\Delta_{t}(\pi_t^{\rm TS})^2}{2v_{t}(\pi_t^{\rm TS})},
\]
where $\pi_t^{\rm TS}$ is the action-sampling distribution of Thompson sampling at time $t$. 

As a result, to show $\Psi_{t}(\pi_t) \leq \lambda/2$, it's enough to show $\Delta_{t}(\pi_{t}^{\rm TS})^2 \leq \lambda v_{t}(\pi_{t}^{\rm TS})$. We show that this holds always for $\lambda = |\A|$, and then show it holds for $\lambda=d$ when $\A \subset \mathbb{R}^d$, $\Theta \subset \mathbb{R}^d$,  and $\E\left[ R_{t,a}| \theta  \right] = a^T \theta$  for all $a\in \A$. 

Recall that by definition, $\pi_{t}^{\rm TS}(a) = \Prob_{t}(A^* =a)$ for each $a\in \A$. Therefore 
\begin{eqnarray}\nonumber
\Delta_{t}(\pi_{t}^{\rm TS}) &=& \E_{t}[R_{t, A^*}] - \sum_{a\in\A} \pi_{t}^{\rm TS}(a) \E_{t}[R_{t,a}] \\ 
\nonumber &=& \sum_{a^* \in \A} \Prob_{t}(A^* = a^*)\E[R_{t,a^*}|A^*=a^*] -  \sum_{a \in \A} \Prob_{t}(A^* = a)\E_{t}[R_{t,a}]\\\label{eq: rewriting TS regret}
&=&  \sum_{a \in \A} \Prob_{t}(A^* = a)\left( \E_{t}[R_{t,a} | A^*=a]- \E_{t}[R_{t,a}]\right)
\end{eqnarray}
and 
\begin{eqnarray} \nonumber
v_{t}(\pi_{t}^{\rm TS}) &=& \sum_{a\in\A} \pi_{t}^{\rm TS}(a) {\rm Var}_{t}(\E[R_{t,a} | A^*]) \\
\nonumber &=&  \sum_{a\in\A} \pi_{t}^{\rm TS}(a) \sum_{a^* \in \A} \Prob_{t}(A^* = a^*)\left( \E_{t}[R_{t,a} | A^*=a^*]- \E_{t}[R_{t,a}]\right)^2 \\ \label{eq: rewriting TS variance}
&=&  \sum_{a,a^*\in\A} \Prob_{t}(A^* = a) \Prob_{t}(A^* = a^*)\left( \E_{t}[R_{t,a} | A^*=a^*]- \E_{t}[R_{t,a}]\right)^2. 
\end{eqnarray} 
{\bf Proof part 1:} 
By the Cauchy-Schwartz inequality, we conclude 
\begin{eqnarray*}
\Delta_{t}(\pi_{t}^{\rm TS})^2 &=&  \left(\sum_{a \in \A} \Prob_{t}(A^* = a)\left( \E_{t}[R_{t,a} | A^*=a]- \E_{t}[R_{t,a}]\right) \right)^2\\ 
&\leq& |\A| \sum_{a \in \A} \Prob_{t}(A^* = a)^2\left( \E_{t}[R_{t,a} | A^*=a]- \E_{t}[R_{t,a}]\right)^2\\
&\leq& |\A| \sum_{a,a' \in \A} \Prob_{t}(A^* = a)\Prob_{t}(A^* = a')\left( \E_{t}[R_{t,a} | A^*=a']- \E_{t}[R_{t,a}]\right)^2\\
&=& |\A| v_{t}(\pi_{t}^{\rm TS}) . 
\end{eqnarray*}
As argued above, this implies $\Psi_{t}(\pi_t) \leq |\A|/2.$   \\
\\
{\bf Proof of part 2:} 
This argument can be extended to provide a tighter bound under a linearity assumption. Now assume    $\A \subset \mathbb{R}^d$, $\Theta \subset \mathbb{R}^d$,  and $\E\left[ R_{t,a}| \theta  \right] = a^T \theta$.  Write $\A = \left\{a_1,\ldots,a_K\right\}$ and define $M \in \mathbb{R}^{K \times K}$ by 
\begin{eqnarray*}
M_{i,j} &=& \sqrt{\Prob_{t}(A^*=a_i) \Prob_{t}(A^*=a_j)}\left(\E_{t}[R_{t,a_i}| A^*=a_j]  -\E_{t}[R_{t,a_i}]   \right) \\
 &=& \sqrt{\alpha_{t}(a_i) \alpha_{t}(a_j)}\left(\E_{t}[R_{t,a_i}| A^*=a_j]  -\E_{t}[R_{t,a_i}]   \right)
\end{eqnarray*}
for all $i,j \in \{1,..,K \}$. Then, by \eqref{eq: rewriting TS regret} and \eqref{eq: rewriting TS variance},
$$ \Delta_{t}(\pi_{t}^{\rm TS}) = \rm{Trace}(M),$$
and
$$  v_{t}(\pi_{t}^{\rm TS}) = \| M \|_{\rm F}^2.$$
This shows, by Fact \ref{fact: trace frobenius inequality} that 
\[ 
\Delta_{t}(\pi_{t}^{\rm TS})^2 \leq  {\rm Rank}(M)v_{t}(\pi_{t}^{\rm TS}) 
\]
We now show ${\rm Rank}(M) \leq d$. Define
\begin{eqnarray*}
\mu &=& \mathbb{E}\left[ \theta| \hist \right] \\
\mu^j &=& \mathbb{E}\left[ \theta| \hist, A^*=a_j \right].
\end{eqnarray*}
Then, by the linearity of the expectation operator, $\E_{t}[R_{t,a_i}| A^*=a_j]  -\E_{t}[R_{t,a_i}]  = (\mu^{j}-\mu)^T a_i$.
Therefore, $M_{i,j} = \sqrt{\alpha_{t}(a_i)\alpha_{t}(a_j)}( (\mu^{j} - \mu)^T a_i)$ and
$$M = \left[\begin{array}{c}
\sqrt{\alpha_{t}(a_1)}\left(\mu^{1}-\mu \right)^{T}\\
\vdots\\
\vdots\\
\sqrt{\alpha_{t}(a_K))}\left(\mu^{k}-\mu \right)^{T}
\end{array}\right]\left[\begin{array}{cccc}
\sqrt{\alpha_{t}(a_1)}a_{1} & \cdots & \cdots & \sqrt{\alpha_{t}(a_K)}a_{K}\end{array}\right].
$$
Since $M$ is the product of a $K$ by $d$ matrix and a $d$ by $K$ matrix, it has rank at most $d$.
\end{proof}

\end{appendix}

\setlength{\bibsep}{3pt}
{
\singlespacing
\bibliography{references}
\bibliographystyle{plainnat}
}

\end{document}